\documentclass{article}

\PassOptionsToPackage{round, compress}{natbib}
\usepackage[final]{neurips_2023}
\usepackage[utf8]{inputenc}
\usepackage[T1]{fontenc}
\usepackage{hyperref}
\usepackage{amsmath}
\usepackage[capitalize,noabbrev]{cleveref}
\usepackage{url}
\usepackage{booktabs}
\usepackage{amssymb}
\usepackage{amsfonts}
\usepackage{amsthm}
\usepackage{bm}
\usepackage{nicefrac}
\usepackage{microtype}
\usepackage[dvipsnames, x11names]{xcolor}
\usepackage{algorithm}
\usepackage{algpseudocode}
\usepackage{graphicx}
\usepackage{multirow}
\usepackage{bbold}
\usepackage{subfig}
\usepackage{ftnxtra}
\usepackage{fnpos}
\usepackage{float}
\usepackage{placeins}
\usepackage{ragged2e}
\usepackage{xspace}
\usepackage{tikz}
\usetikzlibrary{fit,calc}
\usepackage{etoolbox}
\usepackage{nomencl}
\usepackage{thmtools}
\usepackage{mathtools}
\usepackage{cleveref}
\usepackage{thm-restate}
\usepackage{enumitem}
\usepackage{wrapfig}

\usepackage{amsthm}

\newtheorem*{lemma}{Lemma}

\usepackage{amsmath,amsfonts,bm}

\def\eqref#1{equation~\ref{#1}}

\def\1{\bm{1}}

\def\vzero{{\bm{0}}}

\def\vtheta{{\bm{\theta}}}
\def\vbeta{{\bm{\beta}}}
\def\vepsilon{{\bm{\epsilon}}}
\def\vphi{{\bm{\phi}}}

\def\vb{{\bm{b}}}

\def\vu{{\bm{u}}}

\def\vw{{\bm{w}}}
\def\vx{{\bm{x}}}
\def\vy{{\bm{y}}}
\def\vz{{\bm{z}}}

\def\mA{{\bm{A}}}

\def\mI{{\bm{I}}}

\def\mS{{\bm{S}}}

\def\mV{{\bm{V}}}

\def\mX{{\bm{X}}}

\def\mZ{{\bm{Z}}}

\def\mPhi{{\bm{\Phi}}}

\DeclareMathAlphabet{\mathsfit}{\encodingdefault}{\sfdefault}{m}{sl}
\SetMathAlphabet{\mathsfit}{bold}{\encodingdefault}{\sfdefault}{bx}{n}

\def\sE{{\mathbb{E}}}

\def\sP{{\mathbb{P}}}

\def\sR{{\mathbb{R}}}

\def\sV{{\mathbb{V}}}

\def\calA{{\mathcal{A}}}

\def\calC{{\mathcal{C}}}

\def\calG{{\mathcal{G}}}

\def\R{{\mathbb{R}}}
\def\ie{\emph{i.e., }}

\usepackage{etoolbox}

\def \argmax {\mathop{\rm arg\,max}}

\newboolean{showcomments}
\setboolean{showcomments}{true}

\newcommand{\innerproduct}[2]{\langle #1, #2 \rangle}

\newcommand{\regret}{\text{Regret}(T)}
\newcommand{\actions}{\mathcal{A}}
\newcommand{\expect}{\mathbb{E}}

\newcommand{\prob}{\mathbb{P}}

\colorlet{pink}{red!40}

\definecolor{alizarin}{RGB}{227,38,54}
\definecolor{ultramarine}{RGB}{24,13,191}
\definecolor{navyblue}{rgb}{0.0, 0.0, 0.5}

\hypersetup{
  linkcolor  = alizarin,
  citecolor  = blue,
  urlcolor   = ultramarine,
  colorlinks = true,
}

\makenomenclature

\AtBeginEnvironment{proof}{\color{black}}
\newcommand*\circled[1]{\tikz[baseline=(char.base)]{
            \node[shape=circle,draw,inner sep=2pt] (char) {#1};}}

\usepackage[most]{tcolorbox}
\tcbset{textmarker/.style={%
        enhanced,
        parbox=false,boxrule=0mm,boxsep=0mm,arc=0mm,
        outer arc=0mm,left=0mm,right=3mm,top=7pt,bottom=7pt,
        toptitle=1mm,bottomtitle=1mm}}
\newtcolorbox{hintBox}{textmarker, borderline west={6pt}{0pt}{yellow},
              colback=yellow!10!white}
\newtcolorbox{importantBox}{textmarker, borderline west={6pt}{0pt}{pink},
              colback=pink!10!white}
\newtcolorbox{noteBox}{textmarker, borderline west={6pt}{0pt}{Green3},
              colback=Green3!10!white}

\algnewcommand{\LineComment}[1]{\Statex \(\quad \ \ \textcolor{LightSteelBlue3}{\triangledown} \quad \) \textcolor{LightSteelBlue3}{$/^{\star}$ #1 $^{\star}/$}}

\algnewcommand{\NoIndLineComment}[1]{\Statex \(\textcolor{LightSteelBlue3}{\triangledown} \quad \) \textcolor{LightSteelBlue3}{$/^{\star}$ #1 $^{\star}/$}}

\usepackage{silence}
\newcommand{\alg}[1]{{poLinUCB}}

\newcommand{\revision}[1]{{\color{black}  #1}}

 \title{Follow-ups Also Matter: Improving Contextual Bandits    via Post-serving Contexts} 
\author{
    Chaoqi Wang$^{1}$ ~~
    Ziyu Ye$^{1}$ ~~
    Zhe Feng$^2$ ~~
    Ashwinkumar Badanidiyuru$^3$ ~~
    Haifeng Xu$^1$ \\
    University of Chicago$^1$  ~~~~~~
    Google Research$^2$  ~~~~~~
    Google$^3$ \\
      \texttt{\{chaoqi, ziyuye, haifengxu\}@uchicago.edu} \\
     \texttt{\{zhef, ashwinkumarbv\}@google.com} 
}

\begin{document}

\maketitle
\begin{abstract}
Standard contextual bandit problem assumes that all the relevant contexts are observed before the algorithm chooses an arm. This modeling paradigm, while useful, often falls short when dealing with problems in which  valuable additional context can be observed after arm selection. For example, content recommendation platforms like Youtube, Instagram, Tiktok also observe valuable follow-up information pertinent to the user's reward after recommendation (e.g., how long the user stayed, what is the user's watch speed, etc.). To improve online learning efficiency in these applications,  we  study a novel contextual bandit problem with post-serving contexts and design a new algorithm, \alg{},  that achieves tight regret under standard assumptions. Core to our technical proof is a robustified and generalized version of the well-known Elliptical Potential Lemma (EPL), which can accommodate  noise in data. Such robustification is necessary for tackling our problem, and we believe it could also be of general interest.  %
Extensive empirical tests on both synthetic and real-world datasets  demonstrate the significant benefit of utilizing post-serving contexts   as well as the superior performance of  our   algorithm over the state-of-the-art approaches.

\end{abstract}

\section{Introduction}\label{sec:intro}

Contextual bandits represent a fundamental mathematical model that is employed across a variety of applications, such as personalized recommendations~\citep{li2010contextual,wu2016contextual} and online advertising \citep{schwartz2017customer,nuara2018combinatorial}. In their conventional setup, at each round $t$, a learner observes the context $\vx_t$, selects an arm $a_t \in \calA$, and subsequently, observes its associated reward   $r_{t, a_t}$. Despite being a basic and influential   framework, it may not always capture the complexity of real-world scenarios \citep{wang2016learning,yang2020contextual}. Specifically, the learner often observes valuable follow-up information   pertinent to the payoff    post arm selection~(henceforth, the \emph{post-serving context}). Standard contextual bandits framework that neglects such post-serving contexts may result in significantly suboptimal performance due to model misspecification.

Consider an algorithm designed to recommend educational resources to a user by utilizing the user's partially completed coursework, interests, and proficiency as pre-serving context   (exemplified in platforms such as  Coursera). After completing the recommendation, the   system can refine the user's profile by incorporating many post-serving context features such as   course completion status, how much time spent on different educational resources, performances, etc. This transition naturally delineates a mapping from user attributes (i.e., the pre-serving context) to user's learning experiences and outcomes (i.e., the post-serving context).  It is not difficult to see that similar scenarios happen in many other recommender system applications. For instance, in e-commerce platforms (e.g., Amazon, Etsy or any retailing website), the system will first recommend products based on the user's profile information, purchasing pattern and browsing history, etc.; post recommendations, the system can then update these information by integrating   post-serving contexts such as this recent purchase behaviour and product reviews.  Similarly, media content recommendation platforms like Youtube, Instagram and Tiktok, also observe many  post-serving features (e.g., how long the user stayed) that can refine the system's estimation about users' interaction behavior as well as the rewards.

A common salient point in all the aforementioned scenarios  is that the post-serving context are prevalent in many recommender systems; moreover,  despite being unseen during the recommendation/serving phase, they can be estimated from the pre-serving context given enough past data. More formaly, we assume that there exists a learnable mapping $\phi^\star(\cdot): \mathbb{R}^{d_x} \rightarrow \mathbb{R}^{d_z}$ that maps pre-serving feature $\vx \in \mathbb{R}^{d_x} $ to the expectation of the post-serving feature $\vz \in \mathbb{R}^{d_z}  $, i.e.,    $\expect[\vz|\vx] = \phi^\star(\vx)$. 

Unsurprisingly, and as we will also show, integrating the estimation of such post-serving features can significantly help to enhance the performance of contextual bandits. However, most of the existing contextual bandit algorithms, {e.g.,}~\citep{auer2002using, li2010contextual, chu2011contextual, agarwal2014taming, tewari2017ads}, are not designed to accommodate the situations with post-serving contexts. We observe that directly applying these algorithms by ignoring post-serving contexts may lead to linear regret, whereas simple modification of these algorithms will also be sub-optimal. To address these shortcomings, this work introduces a novel algorithm, \alg{}. Our algorithm leverages   historical data to simultaneously estimate reward parameters and the functional mapping from the pre- to post-serving contexts  so to optimize arm selection and achieves sublinear regret. En route to analyzing our algorithm, we also developed new technique tools that may  be of independent interest.

\paragraph{Main Contributions.} %
\begin{itemize}%
    \vspace{-0.1cm}\item First, we introduce a new family of contextual linear bandit problems. In this framework, the decision-making process can effectively integrate post-serving contexts, premised on the assumption that the expectation of   post-serving context as a function of  the pre-serving context  can be gradually learned from historical data. This new model allows us to develop more effective learning algorithms in many natural applications with post-serving contexts. 
    \item Second, to study this new model, we developed a robustified and generalized version of the well-regarded elliptical potential lemma (EPL) in order to accommodate random noise  in the post-serving contexts. While this generalized EPL is an instrumental tool in our algorithmic study, we believe it  is also of independent interest due to the broad applicability of EPL in online learning.
    \item Third, building upon the generalized EPL, we design a new algorithm  \alg{} and prove that it enjoys a regret bound  $ \widetilde{\mathcal{O}}(T^{1-\alpha}d_u^{\alpha} + d_u\sqrt{T K })$, where  $T$ denotes the time horizon and  $\alpha \in [0, 1/2]$ is the learning speed of the pre- to post-context mapping function  $\phi^\star(\cdot)$, { whereas  $d_u = d_x+d_z$ and $K$ denote the parameter dimension and number of arms. When $\phi^\star(\cdot)$ is easy to learn, e.g., $\alpha = 1/2$,  the regret bound becomes $ \widetilde{\mathcal{O}}(\sqrt{Td_u} + d_u\sqrt{T K })$ and is tight. For general functions $\phi^\star(\cdot)$  that satisfy $\alpha \leq 1/2$, this regret bound degrades gracefully as the function  becomes more difficult to learn, i.e., as $\alpha$ decreases.     } 
    
    \item Lastly, we empirically validate our proposed algorithm through thorough numerical experiments on both simulated benchmarks and real-world datasets. The results  demonstrate that our algorithm surpasses existing state-of-the-art solutions. Furthermore, they highlight the tangible benefits of incorporating the functional relationship between pre- and post-serving contexts into the model, thereby affirming the effectiveness of our modeling.
\end{itemize}

\section{Related Works}

\paragraph{Contextual bandits. } 
\revision{
The literature on  linear (contextual) bandits is extensive, with a rich body of works~\citep{abe2003reinforcement,auer2002using,dani2008stochastic,rusmevichientong2010linearly,lu2010contextual, filippi2010parametric,li2010contextual,chu2011contextual,abbasi2011improved, li2017provably, jun2017scalable}. One of the leading design paradigm is to  employ upper confidence bounds as a means of balancing exploration and exploitation, leading to the attainment of minimax optimal regret bounds. The derivation of these regret bounds principally hinges on the utilization of confidence ellipsoids and the elliptical potential lemma. Almost all these works assume that the contextual information governing the payoff is fully observable. In contrast, our work focuses on scenarios where the context is not completely observable during arm selection, thereby presenting new challenges in addressing  partially available information.
\vspace{-0.25cm}
\paragraph{Contextual bandits with partial information. } Contextual bandits with partial information has been relatively limited in the literature. Initial progress in this area was made by \citet{wang2016learning}, who studied settings with hidden contexts. In their setup there is   some context (the post-serving context in our model) that can never by observed by the learner, whereas in our setup the learner can observe post-serving context but only after pulling the arm.  Under the  assumption that if the parameter initialization is extremely close to the true optimal parameter, then they develop  a sub-linear regret algorithm. Our algorithm does not need such strong assumption on parameter initialization. Moreover, we    show that their approach may perform poorly in our setup. %
Subsequent research by \citet{qi2018bandit, yang2020contextual, park2021analysis, yang2021robust, zhu2022robust} investigated scenarios with noisy or unobservable contexts. In these studies, the learning algorithm was designed to predict context information online through context history analysis, or selectively request context data from an external expert. Our work, on the other hand, introduces a novel problem setting that separates contexts into pre-serving and post-serving categories, enabling the exploration of a wide range of problems with varying learnability. Additionally, we also need to employ new techniques for analyzing our problem to get a near-optimal regret bound.
\vspace{-0.25cm}
\paragraph{Generalizations of the elliptical potential lemma~(EPL). } The EPL, introduced in the seminal work of\citet{lai1982least}, is arguably a cornerstone in analyzing how fast stochastic uncertainty decreases with the observations of  new sampled directions. Initially being employed in the analysis of stochastic linear regression, the EPL has since been extensively utilized in stochastic linear bandit problems~\citep{auer2002using,dani2008stochastic,chu2011contextual,abbasi2011improved,li2019nearly,zhou2020neural,wang2022linear}. Researchers have also proposed various generalizations of the EPL to accommodate diverse assumptions and problems. For example, \cite{carpentier2020elliptical}  extended the EPL by allowing for the use of the $ \mX_t^{-p}$-norm, as opposed to the traditional $ \mX_t^{-1}$-norm. Meanwhile, \cite{hamidi2022elliptical} investigated a generalized form of the $1 \wedge\left\|\varphi(\vx_t)\right\|_{\mX_{t-1}^{-1}}^2$ term, which was inspired by the pursuit of variance reduction in non-Gaussian linear regression models. However, existing (generalized) EPLs are inadequate for the analysis of our new problem setup. Towards that end, we develop a new generalization of the EPL in this work to accommodate \emph{noisy feature vectors}. %
}

\section{Linear Bandits with Post-Serving Contexts}\label{sec:prob}

\paragraph{Basic setup.} We hereby delineate a basic setup of linear contextual bandits within the scope of the partial information setting, {whereas multiple generalizations of our framework can be found in Section \ref{sec:extensions}}.  This setting involves a finite and discrete action space, represented as $\calA = [K]$. Departing from the classic contextual bandit setup, the context in our model is bifurcated into two distinct components: the \emph{pre-serving} context,  denoted as $\vx \in \mathbb{R}^{d_x}$, and the \emph{post-serving} context, signified as $\vz \in \mathbb{R}^{d_z}$.  {When it is clear from context, we sometimes refer to pre-serving  context simply as \emph{context} as in classic setup, but always retain the post-serving context notion to emphasize its difference.}  We will denote   $\mX_{t} = \sum_{s=1}^t \vx_s\vx_s^\top + \lambda \mI$ and $\mZ_t = \sum_{s=1}^t \vz_s \vz_s^\top + \lambda \mI$. For the sake of brevity, we employ $\vu = (\vx, \vz)$ to symbolize the stacked vector of $\vx$ and $\vz$, with $d_u = d_x+d_z$ and $\|\vu\|_2\leq L_u$.
The pre-serving context is available during arm selection, while the post-serving context is disclosed \emph{post} the arm selection.
For each arm $a \in \calA$, the payoff, $r_{a}(\vx, \vz)$, is delineated as follows:
\begin{align*}
r_{a}(\vx, \vz) = \vx^\top \vtheta_a^\star + \vz^\top \vbeta_a^\star + \eta,
\end{align*}
where $\vtheta_a^\star$ and $\vbeta_a^\star$ represent the parameters associated with the arm, unknown to the learner, whereas $\eta$ is a random noise sampled from an $R_{\eta}$-sub-Gaussian distribution. We use $\|\vx\|_p$ to denote the $p$-norm of a vector $\vx$, and $\|\vx\|_{\mA}\coloneqq \sqrt{\vx^\top\mA\vx}$ is the matrix norm. For convenience, we assume $\|\vtheta_a^\star\|_2\leq 1$ and $\|\vbeta_{a}^\star\|_2\leq 1$ for all $a \in \mathcal{A}$. Additionally, we posit that the norm of the pre-serving and post-serving contexts satisfies $\|\vx\|_2\leq L_x$ and $\|\vz\|_2 \leq L_z$, respectively, and $\max_{t\in [T]} \sup_{a, b \in \calA} \langle \vtheta_a^{\star} - \vtheta_b^{\star}, \vx_t \rangle \leq 1$ and $\max_{t\in [T]} \sup_{a, b \in \calA}\langle \vbeta_a^{\star} - \vbeta_b^{\star}, \vz_t \rangle \leq 1$, same as in \citep{lattimore2020bandit}.

\subsection{Problem Settings and Assumptions.} The learning process proceeds as follows at each time step $t = 1, 2, \cdots, T$:\vspace{-5pt}
\begin{enumerate}[itemsep=1pt, topsep=3pt, partopsep=0pt, parsep=0pt, leftmargin=*, labelindent=25pt, labelwidth=1em, labelsep=0.5em]
\item The learner observes the context $\vx_t$.
\item An arm $a_t \in [K] $ is selected by the learner.
\item The learner observes the realized reward $r_{t,a_t}$ and  the post-serving context, $\vz_t$.
\end{enumerate}
 Without incorporating the post-serving context, one may incur linear regret as a result of model misspecification, as illustrated in the following observation. To see this, consider a setup with two arms, $a_1$ and $a_2$, and a context $x\in \mathbb{R}$ drawn uniformly from the set $\{-3, -1, 1\}$ with $\phi^\star(x)=x^2$. The reward functions for the arms are noiseless and determined as $r_{a_1}(x) = x + {x^2}/2$ and $r_{a_2}(x) = -x - {x^2}/2$. It can be observed that $r_{a_1}(x) > r_{a_2}(x) \text{ when } x \in \{-3, 1\} \text{ and } r_{a_1}(x) < r_{a_2}(x) \text{ when } x = -1$. Any linear bandit algorithm that solely dependent on the context $x$~(ignoring $\phi^\star(x)$) will inevitably suffer from linear regret, since it is impossible to have a linear function~(i.e., $r(x)=\theta x$) that satisfies the above two   inequalities simultaneously.

\begin{restatable}[]{obs}{propModelStrength} 
There exists linear bandit environments in which  any online algorithm without using post-serving context information will have $\Omega(T)$ regret.  
\end{restatable}

Therefore, it is imperative that an effective learning algorithm must leverage the post-serving context, denoted as $\vz_t$. As one might anticipate, in the absence of any relationship between $\vz_t$ and $\vx_t$, it would be unfeasible to extrapolate any information regarding $\vz_t$ while deciding which arm to pull, a point at which only $\vx_t$ is known. Consequently, it is reasonable to hypothesize a correlation between $\vz_t$ and $\vx_t$. This relationship is codified in the subsequent learnability assumption.

Specifically, we make the following natural assumption  --- there exists an algorithm that can learn the mean of the post-serving context $\vz_t$, conditioned on  $\vx_t$. Our analysis will be general enough to accommodate different convergence rates of the learning algorithm, as one would naturally expect, the corresponding regret will  degrade as this learning algorithm's convergence rate becomes worse.     %
More specifically, we posit that, given the context $\vx_t$, the post-serving context $\vz_t$ is generated as\footnote{We remark that an alternative   view of this reward generation is to treat it as a function of $\vx$ as follows: $r_a(\vx) = \vx^\top \vtheta_a^\star +  \phi ^{\star}(\vx)^\top \vbeta_a^\star + \text{noise}$. Our algorithm can be viewed as a two-phase learning of this structured function, i.e., using $(\vx_t, \vz_t)$ to learn $\phi^{\star}$ that is shared among all arms and then using learned $\hat{\phi}$ to estimate each arm $a$'s reward parameters. This more effectively utilizes data than approaches that directly learns the $r_a(\vx)$ function, as shown in our later experimental section (e.g., see Figure \ref{fig:synthetic-experiments}).} 
\begin{align*}
 \text{post-serving context generation process: } \quad    \vz_t = \phi ^{\star}(\vx_t) + \vepsilon_t, \, \, \ie  \phi^\star(\vx) = \expect[\vz|\vx]. \end{align*} 
 Here, $\vepsilon_t$ is a zero-mean noise vector in $\sR^{d_z}$, and 
$\phi ^{\star}: \mathbb{R}^{d_x} \rightarrow \mathbb{R}^{d_z}$ can be viewed as the post-serving context generating function, which is unknown to the learner. However, we assume $\phi ^{\star}$ is learnable in the following sense. 

\begin{restatable}[Generalized learnability of $\phi^*$]{assump}{learnabilityOnPhiAdvNew}\label{assump:learnabilityOnPhiAdvNew}
There exists an algorithm that, given $t$ pairs of examples $\{(\vx_s, \vz_s)\}_{s=1}^t$ with  arbitrarily chosen $\vx_s$'s,   outputs an estimated function of $\phi^\star: \mathbb{R}^{d_x} \rightarrow \mathbb{R}^{d_z}$ such that for any $\vx\in \mathbb{R}^{d_x}$,  the following holds with probability at least $1-\delta$, 
\begin{align*}
     e_t^\delta \coloneqq \left\|\widehat{\phi}_t(\vx) - \phi^\star(\vx)\right\|_2 \leq C_0 \cdot \left(\|\vx\|_{{\mX^{-1}_{t}}}^2\right)^\alpha \cdot \log\left({t}/{\delta}\right) ,
\end{align*}
where $\alpha \in (0, 1/2]$ and $C_0$ is some universal constant. 
\end{restatable}
The aforementioned assumption encompasses a wide range of learning scenarios, each with different rates of convergence. Generally, the value of $\alpha$ is directly proportional to the speed of learning; the larger the value of $\alpha$, the quicker the learning rate. Later, we will demonstrate that the regret of our algorithm is proportional to $O(T^{1-\alpha})$, exhibiting a graceful degradation as $\alpha$ decreases. The ensuing proposition demonstrates that for linear functions, $\alpha=1/2$. This represents the best learning rate that can be accommodated\footnote{To be precise, our analysis can extend to instances where $\alpha > 1/2$. However, this would not enhance the regret bound since the regret is already $\Omega(\sqrt{T})$, even with knowledge of $\vz$. This would only further complicate our notation, and therefore, such situations are not explicitly examined in this paper.}. In this scenario, the regret of our algorithm is $O(\sqrt{T})$, aligning with the situation devoid of post-serving contexts \citep{li2010contextual,abbasi2011improved}.
\begin{restatable}[]{obs}{propJustificationLinearAssump} 
Suppose $\phi(\cdot)$ is a linear function, i.e., $\phi(\vx) = \mPhi^\top \vx $ for some $\mPhi \in \mathbb{R}^{d_x \times d_z}$, then $e_t^\delta = \mathcal{O}\left(  \|\vx\|_{\mX^{-1}_{t}} \cdot \log\left({t}/{\delta}\right) \right)$.
\end{restatable} 
\vspace{-0.3cm}
This observation   follows from the following inequalities  
$ 
   \left\| \phi_t(\vx) - \phi^\star(\vx)\right\|  =   \| \widehat{\mPhi}_t^\top \vx - {\mPhi^\star}^\top \vx  \|  \leq  \| \widehat{\mPhi}_t - {\mPhi^\star} \|_{\mX_{t}} \cdot \left \| \vx \right\|_{\mX^{-1}_{t}}  
    = \mathcal{O} ( \left \| \vx \right\|_{\mX^{-1}_{t}}  \cdot \log\left(\frac{t}{\delta}\right) )$,
where the last equation is due to the confidence ellipsoid bound \citep{abbasi2011improved}. \revision{We refer curious readers to Appendix~\ref{app:extra-disucssions} for a discussion on other more    challenging $\phi(\cdot)$ functions with possibly worse learning rates $\alpha$.}

\subsection{Warm-up: Why Natural Attempts May Be Inadequate?}\label{sec:preliminary-analysis}
{Given the learnability assumption of $\phi^\star$, one natural idea for solving the above problem is to estimate $\phi^\star$, and then run the standard LinUCB algorithm to estimate $(\vtheta_a, \vbeta_a)$ together by treating  $(\vx_t, \widehat{\phi}_t(\vx_t))$ as the true contexts. Indeed, this is the approach adopted by \citet{wang2016learning} for addressing a similar problem of missing contexts $\vz_t$, except that they used a different unsurprised-learning-based approach to estimate the context $\vz_t$ due to not being able to observing any data about $\vz_t$.} 
Given the estimation of $\widehat{\phi}$,  their algorithm --- which we term it as \emph{LinUCB~}($\widehat{\phi}$) --- iteratively carries out the steps  below at each iteration $t$ (see  Algorithm \ref{alg:ucb-phi-hat} for additional details): 1) Estimation of the context-generating function $\widehat{\phi}_t(\cdot)$ from historical data;  
2) Solve of the following regularized least square problem for each arm $a \in \calA$, with regularization coefficient   $\lambda \geq 0$:
    \begin{align}
    \ell_t (\vtheta_a, \vbeta_a) = \sum_{s\in [t]: a_s = a}  \left(r_{s, a} - \vx_t^\top \vtheta_{a} - \textcolor{red}{\widehat{\phi}_{s}(\vx_s)}^\top \vbeta_{a}\right)^2+\lambda \left( \|\vtheta_a\|_2^2 +  \|\vbeta_a\|_2^2 \right),\label{eq:hatphi}
\end{align}
Under the  assumption that the initialized parameters in their estimations are very close to the global optimum, \citet{wang2016learning} were able to show the $O(\sqrt{T})$ regret of this algorithm. However, it turns out that this algorithm  will fail to yield an satisfying regret bound without their strong assumption on very close parameter initialization, because the errors arising from $\widehat{\phi}(\cdot)$ will significantly enlarge the confidence set of $\widehat{\vtheta}_a$ and $\widehat{\vbeta}_a$.\footnote{Specifically,    \citet{wang2016learning} assume that the initial estimation of {\small $\widehat{\phi}(\cdot)$} is already very close to {\small $\widehat{\phi}^*(\cdot)$}   such that the error arising from {\small $\widehat{\phi}(\cdot)$} diminishes exponentially fast due to the local convergence property of alternating least squares algorithm~\citep{uschmajew2012local}. This strong assumption avoids significant expansion of the confidence sets, but is less realistic in applications so we do not impose such assumption.}  
Thus after removing their initialization assumption,    the best possible regret bound we can possibly achieve is of   order  $\widetilde{\mathcal{O}}({T^{3/4}})$, as illustrated in the subsequent proposition.

\begin{restatable}[Regret of LinUCB-($\widehat{\phi}$)]{prop}{thmRegretPLINUCBPhi}\label{thm:regret-pLinUCBPhi}
The regret of LinUCB-($\widehat{\phi}$) in Algorithm~\ref{alg:ucb-phi-hat} is upper bounded by $\widetilde{\mathcal{O}}\left(  T^{1-\alpha}d_u^\alpha + T^{1-\alpha/2}\sqrt{K d_u^{1+\alpha}} \right)$ with probability at least $1-\delta$, by carefully setting the regularization coefficient $\lambda=\Theta(L_ud_u^\alpha T^{1-\alpha} \log\left({T}/{\delta}\right))$ in Equation~\ref{eq:hatphi}.
\end{restatable} %

Since $\alpha \in [0, 1/2]$, the best possible regret upper bound above is   $\widetilde{\mathcal{O}}(T^{3/4})$, which  is considerably inferior to the sought-after regret bound of $\widetilde{\mathcal{O}}(\sqrt{T})$. 
{Such deficiency of LinUCB-($\widehat{\phi}$) is further observed in all our experiments in  Section \ref{sec:experiments}  as well. These motivate our following  design of a new online learning algorithm to address the challenge of post-serving context, during which we also developed a new technical tool which may be of independent interest to the research community.}

\section{A Robustified and Generalized Elliptical Potential Lemma}
It turns out that solving  the learning problem above requires some novel designs; core to these novelties is a robustified and generalized version of  the well-known elliptical potential lemma (EPL), which  may be of independent interest.  This widely used lemma states a fact about a sequence of vectors $\vx_1, \cdots, \vx_T \in \mathbb{R}^d$. Intuitively, it captures the rate of the   sum of  additional information contained in each $\vx_t$, relative to its predecessors  $\vx_1, \cdots, \vx_{t-1}$. Formally,

\begin{lemma}[Original Elliptical Potential Lemma] 
Suppose (1) $\mX_0 \in \mathbb{R}^{d \times d}$ is any positive definite matrix; (2) $\vx_1, \ldots, \vx_T \in \mathbb{R}^d$ is any sequence of vectors; and (3) $ \mX_t=\mX_0+\sum_{s=1}^t \vx_s \vx_s^{\top}$. Then  the following inequality holds
{\small
\begin{align*}
    \sum_{t=1}^T 1 \wedge\left\|\vx_t\right\|_{\mX_{t-1}^{-1}}^2   &\leq  2 \log\left(\frac{\det\mX_T}{\det\mX_0}\right),
\end{align*}  }
where $a \wedge b = \min \{ a, b\}$ is the $\min$ among $a, b \in \mathbb{R}$. 
\end{lemma}   
To address our new contextual bandit setup with post-serving contexts, it turns out that we will need to  robustify and generlaize  the above lemma to accommodate noises in $\vx_t$ vectors and slower learning rates. Specifically, we present the following  variant of the EPL lemma. 
\begin{restatable}[Generalized Elliptical Potential Lemma]{lem}{lemGEPL}\label{lem:gepl} 
Suppose (1) $\mX_0 \in \mathbb{R}^{d \times d}$ is any positive definite matrix; (2) $\vx_1, \ldots, \vx_T \in \mathbb{R}^d$ is a sequence of vectors with bounded $l_2$ norm $\max_t \|\vx_t\| \leq L_x$; (3) $\vepsilon_1, \ldots, \vepsilon_T \in \mathbb{R}^d$ is a sequence of independent (not necessarily identical) bounded zero-mean   noises satisfying  $\max_{t} \|\vepsilon_t\| \leq L_{\epsilon}$ and $   \expect[\vepsilon_t \vepsilon_t^\top] \succcurlyeq \sigma^2_{\epsilon} \mI$ for any $t$; %
and (4) $\widetilde{\mX}_t$ is defined as follows:
\begin{equation*}\label{eq:data-matrix-GEPL}
    \widetilde{\mX}_t=\mX_0+\sum_{s=1}^t  (\vx_s + \vepsilon_s)  (\vx_s + \vepsilon_s)^{\top} \in \mathbb{R}^{d \times d}. 
\end{equation*} 
Then, for any $p \in [0, 1]$, the following inequality holds with probability at least $1-\delta$,
{\small
\begin{align}\label{eq:gepl-bound}
    \!\!\!\!\!\!\!\!\sum_{t=1}^T\left(1 \wedge\left\| \vx_t \right\|_{\widetilde{\mX}_{t-1}^{-1}}^2\right)^p &\leq  2^p T^{1-p}\log^p\left(\frac{\det \mX_T}{\det\mX_0}\right) +  \frac{  8 L_\epsilon^2 (L_\epsilon + L_x)^2 }{\sigma_{\epsilon}^4} \log\left(\frac{32d L_\epsilon^2 (L_\epsilon + L_x)^2 }{\delta \sigma_{\epsilon}^4}\right) 
\end{align}}
\end{restatable}

Note that the second term is independent of time horizon $T$ and only depends on the setup parameters.  Generally, this can be treated as a constant.    Before describing main proof idea of the lemma, we make a few remarks   regarding Lemma \ref{lem:gepl} to highlight the significance of these generalizations.

\begin{enumerate}%
    \item The original Elliptical Potential Lemma (EPL) corresponds to the specific case of $p = 1$, while Lemma \ref{lem:gepl} is applicable for any $p \in [0,1]$. Notably, the $(1-p)$ rate   in the $T^{1-p}$ term of Inequality \ref{eq:gepl-bound} is tight for \emph{every} $p$.  In fact, this rate is tight even for $\vx_t = 1 \in \mathbb{R}, \forall t$ and $\mX_0 = 1 \in \mathbb{R}$ since, under these conditions, $\left\| \vx_t \right\|_{\mX_{t-1}^{-1}}^2 = 1/t$ and, consequently, {\small $\sum_{t=1}^T\left(1 \wedge\left\| \vx_t \right\|_{\mX_{t-1}^{-1}}^2\right)^p = \sum_{t=1}^T t^{-p}$}, yielding a rate of $T^{1-p}$. 
    This additional flexibility gained by allowing a general $p \in [0,1]$ (with the original EPL corresponding to $p = 1$) helps us to accommodate slower convergence rates when learning the mean context from observed noisy contexts, as formalized in Assumption \ref{assump:learnabilityOnPhiAdvNew}.
    \item A crucial distinction between Lemma \ref{lem:gepl} and the original EPL lies in the definition of the noisy data matrix $\widetilde{\mX}_t$ in Equation \ref{eq:data-matrix-GEPL}, which permits noise. However, the measured context vector $\vx_t$ does \emph{not} have noise. This is beneficial in scenarios where a learner observes noisy contexts but seeks to establish an upper bound on the prediction error based on the underlying noise-free context or the mean context. Such situations are not rare  in real applications; our problem of contextual bandits with post-serving contexts is precisely one of such case --- while choosing an arm, we can   estimate the mean post-serving context conditioned on the observable pre-serving context but are only  able to observe the  noisy realization of post-serving contexts after acting.
    \item Other generalized variants of the EPL have been recently proposed and found to be useful in different contexts. For instance, \cite{carpentier2020elliptical} extends the EPL to allow for the $ \mX_t^{-p}$-norm, as opposed to the $ \mX_t^{-1}$-norm, while \cite{hamidi2022elliptical} explores a generalized form of the $1 \wedge\left\|\varphi(\vx_t)\right\|_{\mX_{t-1}^{-1}}^2$ term, which is motivated by variance reduction in non-Gaussian linear regression models. Nevertheless, to the best of our knowledge, our generalized version is novel and has not been identified in prior works.
\end{enumerate}
\begin{proof}[Proof Sketche of Lemma \ref{lem:gepl}] { 
The formal proof of this lemma is involved and deferred to Appendix \ref{appendix:gepl}. 
At a high level, our proof follows procedure for proving the  original EPL. However, to accommodate the noises in the data matrix, we have to introduce new matrix concentration tools to the original (primarily algebraic) proof,  and also identify the right conditions for the argument to go through. A key lemma to our proof is a high probability bound regarding the constructed noisy data matrix $ \widetilde{\mX}_t$ (Lemma \ref{lem:noisy-matrix-psd} in   Appendix \ref{appendix:gepl}) that we derive based on Bernstein's Inequality for matrices under spectral norm \citep{tropp2015introduction}.  We prove that, under mild assumptions on the noise,  {\small $\|\vx_t\|_{\widetilde{\mX}^{-1}_{t-1}}^2 \leq \|\vx_t\|_{{\mX}^{-1}_{t-1}}^2$} with high probability for any $t$.  
Next, we have to apply the union bound and this lemma to show that the above matrix inequality holds for \emph{every} $t \geq 1$ with high probability. Unfortunately, this turns out to not be true  because when $t$ is very small (e.g., $t=1$), the above inequality cannot hold with high probability. Therefore, we have to use the union bound in a carefully tailored way by  excluding all $t$'s  that are smaller than a certain threshold (chosen optimally by solving certain inequalities) and handling these terms with small $t$ separately (which is the reason of the second  $\mathcal{O}(\log(1/\delta))$ term in Inequality \ref{eq:gepl-bound}). Finally, we refine the analysis of sthe standard EPL by allowing the exponent $p$ in $(1 \land \|\vx_t\|_{\widetilde{\mX}^{-1}_{t-1}}^2)^p$ and derive  an upper bound on the sum {\small $\sum_{t=1}^T (1 \land \|\vx_t\|_{\widetilde{\mX}^{-1}_{t-1}}^2)^p$} with high probability. These together yeilds a robustified and generalized version of EPL as in Lemma \ref{lem:gepl}. } 
\end{proof}

\vspace{-0.5cm}
\section{No Regret Learning in Linear   Bandits with Post-Serving Contexts}

\subsection{The Main Algorithm}

In the ensuing section, we introduce our algorithm, \alg{}, designed to enhance linear contextual bandit learning through the incorporation of post-serving contexts and address the issue arose from the algorithm introduced in Section~\ref{sec:preliminary-analysis}. The corresponding pseudo-code is delineated in Algorithm~\ref{alg:ucb-toy}. Unlike the traditional LinUCB algorithm, which solely learns and sustains confidence sets for parameters (i.e., $\widehat{\vbeta}_a$ and $\widehat{\vtheta}_a$ for each $a$), our algorithm also \emph{simultaneously} manages the same for the post-serving context generating function, $\widehat{\phi}(\cdot)$. Below, we expound on our methodology for parameter learning and confidence set construction.

\begin{algorithm}[t]
  \caption{ \alg{}  (\emph{Linear UCB with post-serving contexts}) %
  }
  \label{alg:ucb-toy}
\begin{algorithmic}[1]
    \For{$t=0, 1, \ldots, T$}
    \State Receive the \emph{pre-serving context} $\vx _{t}$
    \State Compute the optimistic parameters by maximizing the UCB objective $${\left(a_t, {\widetilde{\phi}_{t}(\vx_t)}, \widetilde{\vw}_t\right) = {\underset{(a, \phi, \vw_a) \in [K] \times \calC_{t-1}\left(\widehat{\phi}_{t-1}, \vx_t\right) \times \calC_{t-1}(\widehat{\vw}_{t-1,a})  }{\argmax}} { \begin{bmatrix} \vx_t \\ \phi(\vx_t) \end{bmatrix}^\top \vw_{a} }}.$$ 
    \State Play the arm $a_t$ and receive the realized \emph{post-serving context} as $\vz_t$ and  the real-valued  reward $$r_{t, a_t} =   \begin{bmatrix} \vx_t \\ \vz_t \end{bmatrix}^\top \vw_{a_t}^{\star}  + \eta_t. $$
  \State Compute $\widehat{\vw}_{t,a}$ using Equation~\ref{eqn:w-closed-form} for each $a \in \calA$.
  \State Compute the estimated post-serving context generating function $\widehat{\phi}_t(\cdot)$ using ERM. 
    \State Update confidence sets $\calC_t(\widehat{\vw}_{t, a})$ and $\calC_t(\widehat{\phi}_t, \vx_t)$ for each $a$ based on Equations~\ref{eqn:w-confidence-set} and \ref{eqn:phi-confidence-set}.
    \EndFor
\end{algorithmic}
\end{algorithm}
\textbf{Parameter learning}. During each iteration $t$, we  fit the function $\widehat{\phi}_t(\cdot)$ and the parameters $\{\widehat{\vtheta}_{t, a}\}_{a\in\calA}$ and $\{\widehat{\vbeta}_{t, a}\}_{a\in \calA}$. To fit $\widehat{\phi}_t(\cdot)$, resort to the conventional empirical risk minimization~(ERM) framework. As for $\{\widehat{\vtheta}_{t, a}\}_{a\in\calA}$ and $\{\widehat{\vbeta}_{t, a}\}_{a\in \calA}$, we solve the following least squared problem for each arm $a$,
\begin{align}\label{eqn:w-closed-form}
    \ell_t (\vtheta_a, \vbeta_a) &= \sum_{s\in[t]: a_s = a}  \left(r_{s, a} - \vx_s^\top \vtheta_{a} - {\color{red}{\vz_s}}^{\top} \vbeta_{a}\right)^2+\lambda \left(\|\vtheta_a\|_2^2 + \|\vbeta_{a}\|_{2}^{2}\right).
\end{align}
For convenience, we use $\vw$ and $\vu$ to denote $(\vtheta, \vbeta)$ and $(\vx, \vz)$ respectively. The closed-form solutions to $\widehat{\vtheta}_{t, a}$ and $\widehat{\vbeta}_{t, a}$ for each arm $a \in \calA$ are 
\begin{align}
    \widehat{\vw}_{t, a} \coloneqq \begin{bmatrix}
        \widehat{\vtheta}_{t, a} \\
        \widehat{\vbeta}_{t, a}
    \end{bmatrix}= \mA_{t, a}^{-1} \vb_{t, a}
\textnormal{, where }
    \mA_{t, a} = \lambda \mI + \sum_{s: a_s = a}^t \vu_{s}\vu_s^\top
    \quad \text{and}\quad 
    \vb_{t, a} = \sum_{s: a_s = a}^t r_{s, a} \vu_s. \label{eq:new-estimator}
\end{align}
\textbf{Confidence set construction. } At iteration $t$, we construct the confidence set for $\widehat{\phi}_t(\vx_t)$ by
\begin{align}\label{eqn:phi-confidence-set}
    \calC_{t}\left(\widehat{\phi}_t, \vx_t\right) \coloneqq \left\{\vz \in \R^d: \left\|\widehat{\phi}_t(\vx_t) - \vz\right\|_2 \leq e^\delta_t\right\}.
\end{align}
Similarly, we can construct the confidence set for the parameters $\widehat{\vw}_{t, a}$ for each arm $a \in \calA$ by 
\begin{align}\label{eqn:w-confidence-set}
    \calC_t\left(\widehat{\vw}_{t, a}\right) \coloneqq \left\{\vw \in \R^{d_x + d_z}: \left\|\vw - \widehat{\vw}_{t, a}\right\|_{\mA_{t,a}}\leq \zeta_{t,a} \right\},  
\end{align}
where $\zeta_{t,a} = 2\sqrt{\lambda} + R_{\eta}\sqrt{ d_u\log\left((1+n_t(a)L_u^2/\lambda)/\delta\right)}$ and $n_t(a) = \sum_{s=1}^t \mathbb{1}[a_s=a]$. Additionally, we further define $\zeta_{t} \coloneqq \max_{a \in \calA} \zeta_{t,a}$. By the assumption \ref{assump:learnabilityOnPhiAdvNew} and Lemma~\ref{thm:conf-2}, we have the followings hold with probability at least $1-\delta$ for each of the following events,
\begin{align}
    \phi^\star(\vx_t) \in \calC_{t}\left(\widehat{\phi}_t, \vx_t\right) \quad \textnormal{and}\quad  \vw^\star \in \calC_{t}\left(\widehat{\vw}_{t,a}\right).
\end{align}

\subsection{Regret Analysis}

In the forthcoming section, we establish the regret bound. Our proof is predicated upon the conventional proof of LinUCB~\citep{li2010contextual} in conjunction with our robust elliptical potential lemma. The pseudo-regret~\citep{audibert2009exploration} within this partial contextual bandit problem is defined as,
\begin{align}
    &R_T=\regret =  \sum_{t=1}^T \left( r_{t, a^\star_t} - r_{t,a_t} \right), \label{eq:pseudo-regret}
\end{align}
in which we reload the notation of reward by ignoring the noise,
\begin{align}
    r_{t, a} = \langle \vtheta_{a}^\star, \vx_t \rangle  +  \langle \vbeta_{a}^\star, \phi^\star(\vx_t) \rangle \quad \textnormal{and} \quad a^\star_t = \argmax_{a \in \actions}\  \langle \vtheta_{a}^\star, \vx_t \rangle + \langle\vbeta_{a}^\star, {\phi^{\star}(\vx _t)} \rangle. \label{eq:pc-astar}
\end{align}
It is crucial to note that our definition of the optimal action, $a^{\star}_t$, in Eq.~\ref{eq:pc-astar} depends on $\phi^{\star}(\vx _t)$ as opposed to $\vz_t$. This dependency ensures a more pragmatic benchmark, as otherwise, the noise present in $\vz$ would invariably lead to a linear regret, regardless of the algorithm implemented. In the ensuing section, we present our principal theoretical outcomes, which provide an upper bound on the regret of our \alg{} algorithm.

\begin{restatable}[Regret of \alg{}]{thm}{thmRegretPLINUCB}\label{thm:regret-pLinUCB}
The regret of poLinUCB in Algorithm~\ref{alg:ucb-toy} is upper bounded by $ \widetilde{\mathcal{O}}\left(T^{1-\alpha}d_u^{\alpha} + d_u\sqrt{T K }\right)$ with probability at least $1-\delta$, if $T=\Omega(\log(1/\delta))$.
\end{restatable}
{The first term in the bound is implicated by learning the function $\phi^\star(\cdot)$. Conversely, the second term resembles the one derived in conventional contextual linear bandits, with the exception that our dependency on $d_u$ is linear. This linear dependency is a direct consequence of our generalized robust elliptical potential lemma. The proof is deferred in Appendix~\ref{appendix:regret}.}

\section{Generalizations}\label{sec:extensions}
So far we have focused on a basic linear bandit setup with post-serving features. Our results and analysis can be easily generalized to   other variants of linear bandits, including those with feature mappings, and below we highlight some of these generalizations. They use similar proof ideas, up to some technical modifications; we thus defer all their formal proofs to Appendix \ref{appendix:extensitions}.

\subsection{Generalization to Action-Dependent Contexts} 
Our basic setup in Section \ref{sec:prob} has a single context $\vx_t$ at any time step $t$. This can be generalized to   action-dependent contexts  settings as studied in previous works (e.g., \citet{li2010contextual}). That is, during each iteration indexed by $t$, the learning algorithm observes a context $\vx_{t, a}$ for each individual arm $a \in \calA$. Upon executing the action of pulling arm $a_t$, the corresponding post-serving context $\vz_{t, a_t}$ is subsequently revealed. Notwithstanding, the post-serving context for all alternative arms remains unobserved. The entire procedure is the same as that of Section \ref{sec:prob}.

In extending this framework, we persist in our assumption that for each arm $a \in \calA$, there exists a specific function $\phi^\star_{a}(\cdot): \mathbb{R}^{d_x} \rightarrow \mathbb{R}^{d_z}$ that generates the post-serving context $\vz$ upon receiving $\vx$ associated with arm $a \in \calA$. The primary deviation from our preliminary setup lies in the fact that we now require the function $\phi^\star_{a}(\cdot)$ to be learned for each arm independently. The reward is generated as
\begin{align*}
    r_{t, a_t} = \langle \vtheta_{a_t}^\star, \vx_{t, a_t} \rangle + \langle \vbeta_{a_t}^\star, \vz_{t, a_t}\rangle + \eta_t. 
\end{align*}

The following proposition shows our regret bound for this action-dependent context case. Its proof   largely draws upon the proof idea of Theorem~\ref{thm:regret-pLinUCB} and also relies on the generalized EPL Lemma \ref{lem:gepl}.

 \begin{restatable}[]{prop}{propActionDependentContextsRegret}\label{prop:ADC-plinUCB-regret-bound}
The regret of poLinUCB in Algorithm~\ref{alg:ucb-toy} for action-dependent contexts is upper bounded by $ \widetilde{\mathcal{O}}\left(T^{1-\alpha}d_u^{\alpha}\sqrt{K} + d_u\sqrt{T K }\right)$ with probability at least $1-\delta$ if $T=\Omega(\log(1/\delta))$.
 \end{restatable}
{The main difference with the bound in Theorem~\ref{thm:regret-pLinUCB} is the additional $\sqrt{K}$ appeared in the first term, which is caused by learning multiple $\phi^\star_a(\cdot)$ functions with $a \in \calA$.}

\subsection{Generalization to Linear Stochastic Bandits}

Another variant of linear bandits is the \emph{linear stochastic bandits} setup (see, e.g., \citep{abbasi2011improved}). This model allows infinitely many arms, which consists of a decision set $D_t \subseteq \mathbb{R}^d$ at time $t$, and the learner picks an action $\vx_t \in D_t$. %
 This setup naturally generalizes to our problem with post-serving contexts. That is, at iteration $t$, the learner selects an arm $\vx_t \in D_t$ first,   receives   reward $r_{t, \vx_t}$,  and then observe the post-serving feature $\vz_t$ conditioned on $\vx_t$. 
Similarly, we  assume the existence of a mapping $\phi^\star(\vx_t) = \expect[\vz_t|\vx_z]$ that satisfies the Assumption~\ref{assump:learnabilityOnPhiAdvNew}. Consequently, the realized reward is generated as follows where   $\vtheta^*, \vbeta^*$ are unknown parameters: 
\begin{align*}
    r_{t, \vx_t} = \langle \vx_t,  \vtheta^\star \rangle +  \langle \vz_t, \vbeta^\star \rangle + \eta_t.
\end{align*}
Therefore, the learner needs to estimate the linear parameters $\widehat{\vtheta}$ and $\widehat{\vbeta}$, as well as the function $\widehat{\phi}(\cdot)$. We obtain the following proposition for the this setup.
 \begin{restatable}[]{prop}{propAbbasiRegret}\label{prop:abbasi-plinUCB-regret-bound}
The regret of poLinUCB in Algorithm~\ref{alg:ucb-toy} for the above setting is upper bounded by $ \widetilde{\mathcal{O}}\left(T^{1-\alpha}d_u^{\alpha} + d_u\sqrt{T}\right)$ with probability at least $1-\delta$ if $T=\Omega(\log(1/\delta))$.
 \end{restatable}
 \subsection{Generalization to Linear Bandits with Feature Mappings}
 Finally, we briefly remark that while we have so far assumed that the arm parameters are directly linear in the context $\vx_t, \vz_t$, just like classic linear bandits our analysis can be easily generalized to accommodate feature mapping $\pi^x(\vx_t)$ and $\pi^z(\vz_t) = \pi^z(\phi(\vx_t) + \varepsilon_t)$. Specifically, if the reward generation process is $ r_{ a} = \langle \vtheta_{a}^\star, \pi^x(\vx_t)  \rangle + \langle \vbeta_{a}^\star, \pi^z(\vz_{t})\rangle + \eta_t$ instead, then we can simply view $  \Tilde{\vx_t} = \pi^x(\vx_t) $ and $\Tilde{\vz_t} = \pi^z(\vz_t)$ as the new features, with  $\Tilde{\phi}(\vx_t)  = \mathbb{E}_{\vepsilon_t} [\pi^z( \phi(\vx_t) + \vepsilon_t)]$. By working with   $  \Tilde{\vx_t}, \Tilde{\vz_t}  , \Tilde{\phi}$, we shall obtain the same guarantees as Theorem \ref{thm:regret-pLinUCB}.

\section{Experiments}\label{sec:experiments}
This section presents a comprehensive evaluation of our proposed \alg{} algorithm on both synthetic and real-world data, demonstrating its effectiveness in incorporating follow-up information and outperforming the LinUCB($\widehat{\phi}$) variant. \revision{More empirical results can be found in Appendix~\ref{app:extra-experiments}.}

\subsection{Synthetic Data with Ground Truth Models}

\begin{figure}[t]
\vspace{-0.7cm}
    \centering
    \includegraphics[width=0.85\textwidth]{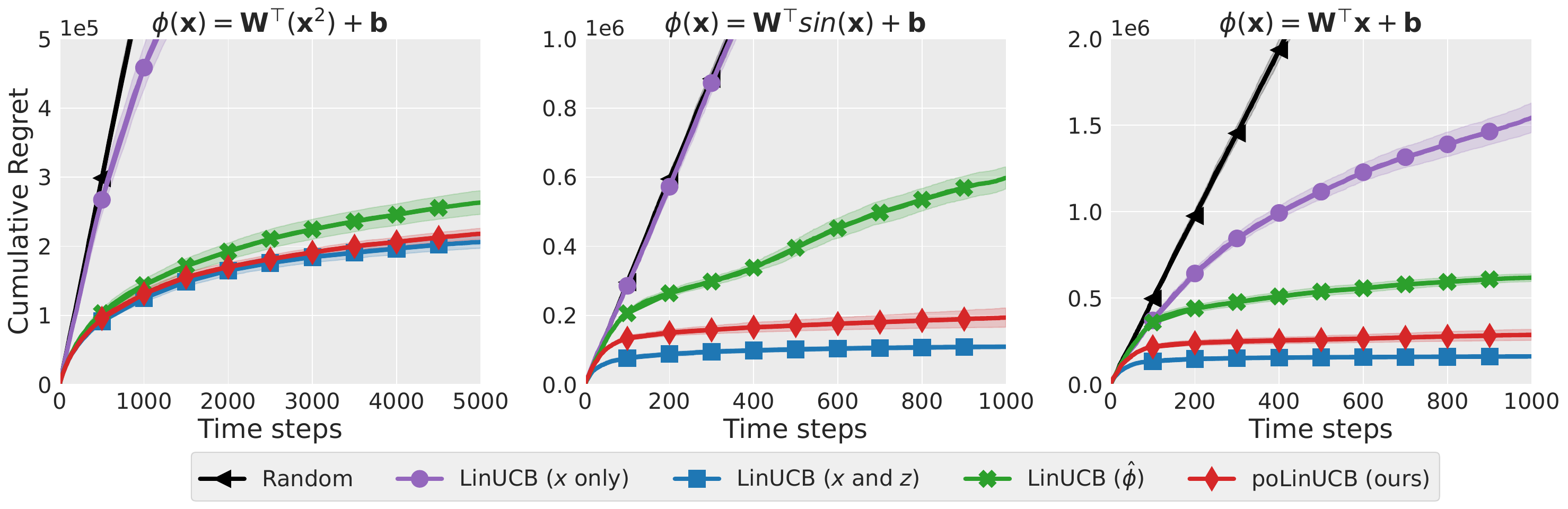}
    \vspace{-0.2cm}
    \caption{Cumulative Regret in three synthetic environments. Comparisons of different algorithms in terms of cumulative regret across the three synthetic environments. Our proposed poLinUCB (ours) consistently outperforms other strategies~(except for LinUCB which has access to the post-serving context during arm selection), showcasing its effectiveness in utilizing post-serving contexts. The shaded area denotes the standard error computed using 10 different random seeds.}
    \vspace{-0.6cm}
    \label{fig:synthetic-experiments}
\end{figure}

\paragraph{Evaluation Setup.} We adopt three different synthetic environments that are representative of a range of mappings from the pre-serving context to the post-serving context: polynomial, periodicical and linear functions. The pre-serving contexts are sampled from a uniform noise in the range $[-10, 10]^{d_x}$, and Gaussian noise is employed for both the post-serving contexts and the rewards. In each environment, the dimensions of the pre-serving context ($d_x$) and the post-serving context ($d_z$) are of $100$ and $5$, respectively with 10 arms ($K$). The evaluation spans $T=1000$ or $5000$ time steps, and each experiment is repeated with $10$ different seeds. The cumulative regret for each policy in each environment is then calculated to provide a compararison.
\vspace{-0.2cm}
\paragraph{Results and Discussion.} Our experimental results, which are presented graphically in Figures~\ref{fig:synthetic-experiments}, provide strong evidence of the superiority of our proposed poLinUCB algorithm. Across all setups, we observe that the LinUCB ($x$ and $z$) strategy, which has access to the post-serving context during arm selection, consistently delivers the best performance, thus serving as the upper bound for comparison. On the other hand, the Random policy, which does not exploit any environment information, performs the worst, serving as the lower bound. Our proposed poLinUCB (ours) outperforms all the other strategies, including the LinUCB ($\hat{\phi}$) variant, in all three setups, showcasing its effectiveness in adaptively handling various mappings from the pre-serving context to the post-serving context. Importantly, poLinUCB delivers significantly superior performance to LinUCB ($x$ only), which operates solely based on the pre-serving context.

\subsection{Real World Data without Ground Truth}
\noindent \textbf{Evaluation Setup.} The evaluation was conducted on a real-world dataset, MovieLens~\citep{harper2015movielens}, where the task is to recommend movies~(arms) to a incoming user~(context). Following~\citet{yao2023bad}, we first map both movies and users to 32-dimensional real vectors using a  neural network trained for predicting the rating.   Initially, $K=5$ movies were randomly sampled to serve as our arms and were held fixed throughout the experiment. The user feature vectors were divided into two parts serving as the pre-serving context~($d_x=25$) and the post-serving context~($d_z=7$). We fit the function $\phi(\vx)$ using a two-layer neural network with 64 hidden units and ReLU activation. The network was trained using the Adam optimizer with a learning rate of 1e-3. At each iteration, we randomly sampled a user from the dataset and exposed only the pre-serving context $\vx$ to our algorithm. The reward was computed as the dot product of the user's feature vector and the selected movie's feature vector and was revealed post the movie selection. The evaluation spanned $T=500$ iterations and repeated with $10$ seeds. 

\begin{wrapfigure}{r}{0.38\textwidth} %
  \centering
  \vspace{-1.2cm}
  \includegraphics[width=0.4\textwidth]{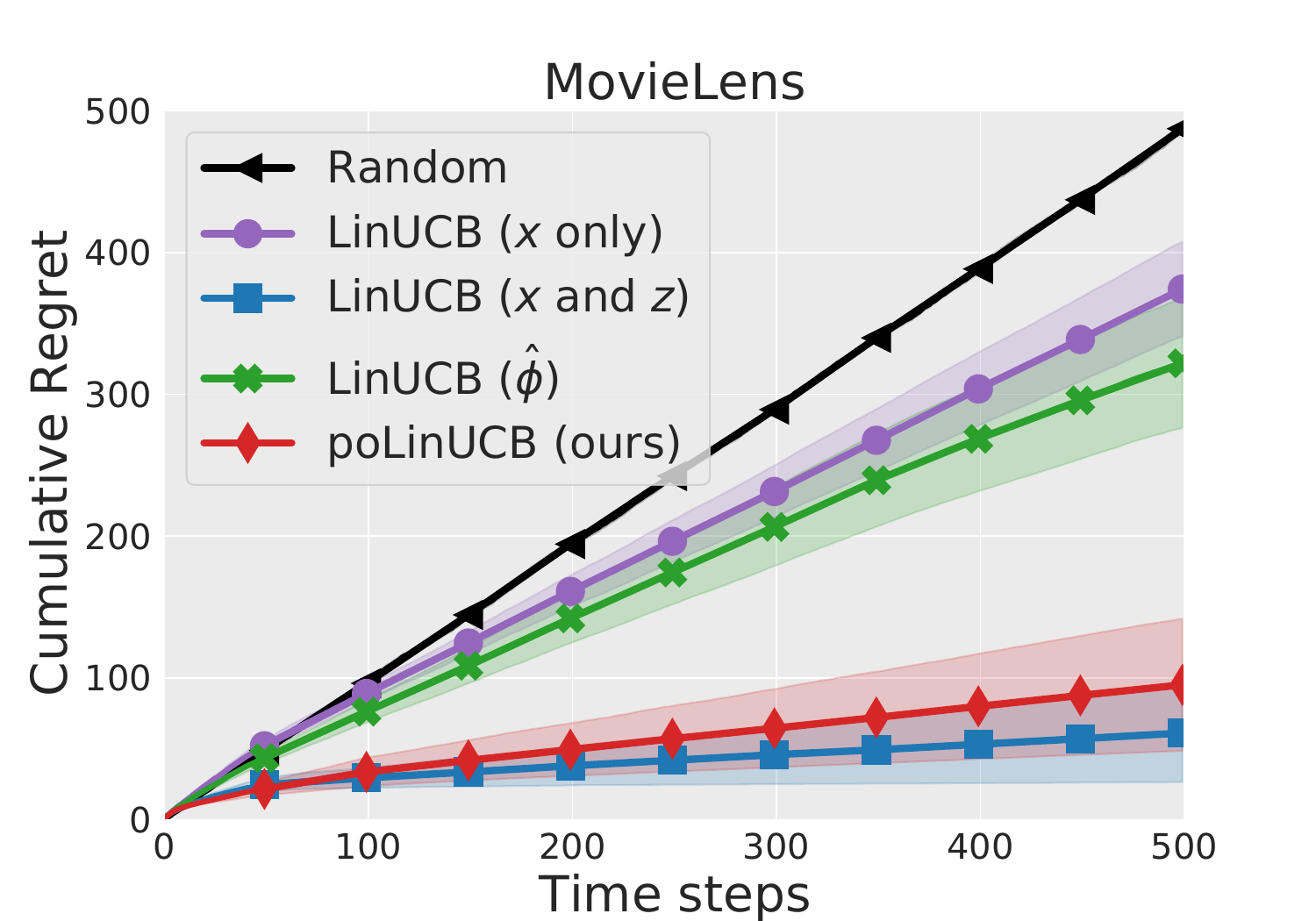} %
  \caption{Results on MovieLens.}
  \vspace{-0.5cm}
  \label{fig:movie-lens}
\end{wrapfigure}
\vspace{-0.2cm}
\paragraph{Results and Discussion.} The experimental results, presented in Figure~\ref{fig:movie-lens}, demonstrate the effectiveness of our proposed algorithm. The overall pattern is similar to it observed in our synthetic experiments. Our proposed policy consistently outperforms the other strategies~(except for LinUCB with both pre-serving and post-serving features). Significantly, our algorithm yields superior performance compared to policies operating solely on the pre-serving context, thereby demonstrating its effectiveness in leveraging the post-serving information.

\section{Conclusions and Limitations}
\revision{
\noindent \textbf{Conlusions.} In this work, we have introduced a novel contextual bandit framework that incorporates post-serving contexts, thereby widening the range of complex real-world challenges it can address. By leveraging historical data, our proposed algorithm, \alg{}, estimates the functional mapping from pre-serving to post-serving contexts, leading to improved online learning efficiency. For the purpose of theoretical analysis, the elliptical potential lemma has been expanded to manage noise within post-serving contexts, a development which may have wider applicability beyond this particular framework. Extensive empirical tests on synthetic and real-world datasets have demonstrated the significant benefits of utilizing post-serving contexts and the superior performance of our algorithm compared to state-of-the-art approaches. 

\noindent \textbf{Limitations.} %
Our theoretical analysis hinges on a crucial assumption that the function $\phi^\star(\cdot)$ is learnable, which may not always be satisfied. This is particularly a concern  when the post-serving contexts may hold additional information that cannot be deduced from the pre-serving context, irrespective of the amount of data collected. In such scenarios, the function mapping from the pre-serving context to the post-serving context may be much more difficult to learn, or even not learnable. Consequently, a linear regret may be inevitable due to model misspecification. However, from a practical point of view, our empirical findings from the real-world MovieLens dataset demonstrate that modeling the functional relationship between the pre-serving and post-serving contexts can still significantly enhance the learning efficiency. We hope our  approaches can inform the design of  practical algorithms that more effectively utilizes post-serving data in  recommendations.  

\noindent \textbf{Acknowledgement.} This work is supported
by an NSF Award CCF-2132506, an Army Research Office Award W911NF-23-1-0030, and an Office of Naval Research Award   N00014-23-1-2802. 
}

\clearpage
\bibliographystyle{plainnat}
\bibliography{ref-bandits,ref-bias}

\newpage
\appendix
\normalsize
\setcounter{footnote}{0}

\begin{appendix}

\section[Constructing Confidence Sets (\textit{with the Predicted Context Generating Function})]{Algorithm and Regret Analysis of LinUCB({\(\widehat{\phi}\)})}

We present the details of the algorithm described in Section~\ref{sec:preliminary-analysis} and the proof of the regret bound. 

\subsection{Main Algorithm}

\paragraph{Parameter learning.} We consider solving the following regularized least squared problem for estimating $\{\widehat{\vtheta}_{t, a}\}_{a\in\calA}$ and $\{\widehat{\vbeta}_{t, a}\}_{a\in \calA}$ for each arm $a$:
\begin{align}\label{eqn:phi-least-squares}
    \ell_t (\vtheta_a, \vbeta_a) = \sum_{s: a_s = a}^t  \left(r_{s, a} - \vx_t^\top \vtheta_{a} - \textcolor{red}{\widehat{\phi}_{s}(\vx_s)}^\top \vbeta_{a}\right)^2+\lambda \left( \|\vtheta_a\|_2^2 +  \|\vbeta_a\|_2^2 \right),
\end{align}
where $\lambda  \geq 0$ are penalty factors ensuring the uniqueness of minimizers $\widehat{\vtheta}_{t, a}$ and $\widehat{\vbeta}_{t, a}$.

In the same convention, we use $\vw$ to denote $\left(\vtheta, \vbeta\right)$, and $\vu$ to denote $\left(\vx, \textcolor{red}{\widehat{\phi}(\vx)}\right)$. The closed-form solutions for $\widehat{\vtheta}_{t, a}$ and $\widehat{\vbeta}_{t, a}$ in this least squared problem then become: 
\begin{align*}
    \widehat{\vw}_{t, a} \coloneqq \begin{bmatrix}
        \widehat{\vtheta}_{t, a} \\
        \widehat{\vbeta}_{t, a}
    \end{bmatrix}= \mA_{t, a}^{-1} \vb_{t, a},
\end{align*}
where we reload the notations of $\mA_{t,a}$ and $\vb_{t, a}$,
\begin{align}
    \mA_{t, a} = \lambda \mI + \sum_{s: a_s = a}^t \vu_{s}\vu_s^\top
    \quad \text{and}\quad 
    \vb_{t, a} = \sum_{s: a_s = a}^t r_{s, a} \vu_s. \label{eq:new-estimator-hat}
\end{align}

\textbf{Confidence set construction. } At iteration $t$, we construct the confidence set for $\widehat{\phi}_t(\vx_t)$ by
\begin{align}\label{eqn:w-confidence-set-hat}
    \calC_{t}\left(\widehat{\phi}_t, \vx_t\right) \coloneqq \left\{\vz \in \R^d: \left\|\widehat{\phi}_t(\vx_t) - \phi^\star(\vx_t)\right\|_2 \leq e^\delta_t\right\}.
\end{align}

The construction of the confidence set for $\widehat{\vw}_{t,a}$ will be different, as we are using the predicted value $\widehat{\phi}_{t}(\cdot)$ for linear regression. Consider the following,
\begin{align*}
    \mA_{t, a}\left( \begin{bmatrix}
        \widehat{\vtheta}_{t, a} \\ 
        \widehat{\vbeta}_{t, a}
    \end{bmatrix} -  \begin{bmatrix}
        \vtheta_{a}^\star \\ 
        \vbeta_{a}^\star
    \end{bmatrix}\right) = &\underbrace{\sum_{s: a_s=a}^t \left(\vepsilon^\top_s \vbeta_{a}^\star \begin{bmatrix}
        \vx_s \\ 
        \widehat{\phi}_s(\vx_s)
    \end{bmatrix}\right) }_{\circled{1}}
    + \underbrace{\sum_{s: a_s = a}^t \left(\left(\phi^\star(\vx_s) - \widehat{\phi}_s(\vx_s)\right)^\top \vbeta_{a}^\star \begin{bmatrix}
        \vx_s \\ 
        \widehat{\phi}_s(\vx_s)
    \end{bmatrix}\right)}_{\circled{2}} \notag \\
    & + \underbrace{\sum_{s: a_s = a}^t \eta_s \begin{bmatrix}
        \vx_s \\ 
        \widehat{\phi}_s(\vx_s)
    \end{bmatrix} }_{\circled{3}}
     - \underbrace{\lambda \begin{bmatrix}
        \vtheta^\star_a \\ 
         \vbeta^\star_a
    \end{bmatrix}}_{\circled{4}}
\end{align*}

 Therefore, the confidence set will be enlarged due to the error introduced by  $\widehat{\phi}_{t}(\cdot)$. In the below, we derive the confidence set. To build the confidence set, we need to bound 
\begin{align*}
   \|\widehat{\vw}_{t,a} - \vw_{t,a}^\star\|_{\mA_{t, a}} = \left \|\begin{bmatrix}
        \widehat{\vtheta}_{t, a} \\ 
        \widehat{\vbeta}_{t, a}
    \end{bmatrix} -  \begin{bmatrix}
        \vtheta_{a}^\star \\ 
        \vbeta_{a}^\star
    \end{bmatrix}\right\|_{\mA_{t, a}} = \left\|\circled{1}+\circled{2}+\circled{3}+\circled{4}\right\|_{\mA_{t, a}^{-1}}
\end{align*}
Since both $\{\vepsilon_s\}_{s=1}^t$ and $\{\eta_s\}_{s=1}^t$ are i.i.d sub-Gaussian random variables, respectively, we can use the self-normalized inequality to bound the corresponding terms, i.e., the followings hold with probability at least $1-2\delta$,
\begin{align*}
     \|\widehat{\vw}_{t,a} - \vw_{t,a}^\star\|_{\mA_{t, a}} &\leq \sqrt{2(L_{\epsilon}^2+R_{\eta}^2) \log\left(\frac{\det(\mA_{t, a})^{1/2} \det(\lambda \mI)^{-1/2}}{\delta/2}\right)} +  \frac{L_u}{\sqrt{\lambda}} \left(\sum_{s=1}^t e^{\delta/t}_s\right) + 2\sqrt{\lambda} \notag \\
     &\leq \sqrt{2(L_{\epsilon}^2+R_{\eta}^2) \log\left(\frac{1+n_t(a)L_{u}^2/\lambda}{\delta/2}\right)} +  \frac{L_u}{\sqrt{\lambda}} \left(\sum_{s=1}^t e^{\delta/t}_s\right) + 2\sqrt{\lambda}
\end{align*}

\begin{algorithm}[t]
  \caption{ LinUCB-$(\widehat{\phi})$  (\emph{Linear UCB adapted from \citet{wang2016learning} with post-serving contexts; The differences with Algorithm~\ref{alg:ucb-toy} are highlighted in {\color{navyblue}blue} color.}) %
  }
  \label{alg:ucb-phi-hat}
\begin{algorithmic}[1]
    \For{$t=0, 1, \ldots, T$}
    \State Receive the \emph{pre-serving context} $\vx _{t}$
    \State Compute the optimistic parameters by maximizing the UCB objective $${\left(a_t, {\widetilde{\phi}_{t}(\vx_t)}, \widetilde{\vw}_t\right) = {\underset{(a, \phi, \vw_a) \in [K] \times \calC_{t-1}\left(\widehat{\phi}_{t-1}, \vx_t\right) \times \calC_{t-1}(\widehat{\vw}_{t-1,a})  }{\argmax}} { \begin{bmatrix} \vx_t \\ \phi(\vx_t) \end{bmatrix}^\top \vw_{a} }}.$$ 
    \State Play the arm $a_t$ and receive the realized \emph{post-serving context} as $\vz_t$ and  the real-valued  reward $$r_{t, a_t} =   \begin{bmatrix} \vx_t \\ \vz_t \end{bmatrix}^\top \vw_{a_t}^{\star}  + \eta_t. $$
  \State {\color{navyblue}Compute the estimated post-serving context generating function $\widehat{\phi}_t(\cdot)$ using ERM. }
  \State {\color{navyblue} Compute $\widehat{\vw}_{t,a}$ by solving Equation~\ref{eqn:phi-least-squares} for each $a$.}
    \State {\color{navyblue} Update confidence sets $\calC_t(\widehat{\vw}_{t, a})$ and $\calC_t(\widehat{\phi}_t, \vx_t)$ for each $a$ based on Equations~\ref{eqn:w-confidence-set-hat} and \ref{eqn:phi-confidence-set}.}
    \EndFor
\end{algorithmic}
\end{algorithm}

Therefore, the confidence set is
\begin{align}
    \calC_{t, a}(\widehat{\vw}_{t,a}) = \left\{ \vw \in \mathbb{R}^{d_u} : \|\vw - \widehat{\vw}_{t, a}\|_{\mA_{t, a}} \leq \zeta_{t,a} \right\},
\end{align}
where $\zeta_{t,a}= \sqrt{2(L_{\epsilon}^2+R_{\eta}^2) \log\left((1+n_t(a)L_{u}^2/\lambda)/(\delta/2)\right)} +  {L_u} \left(\sum_{s=1}^t e^{\delta/t}_s\right)/{\sqrt{\lambda}} + 2\sqrt{\lambda}$. In comparison to the original confidence set, there is one additional term due to the generalization error introduced from $\widehat{\phi}_s(\cdot)$. In the next section, we will provide a regret analysis, which following from the proof of LinUCB~\citep{li2010contextual}. We simply have the following
\begin{align*}
    &R_T = \sum_{t=1}^T \left( r_{t, a_t^\star} - r_{t, a_t}\right) =\sum_{t=1}^T \Delta_t \leq \sqrt{T\sum_{t=1}^T \Delta_t^2}\\
    &\leq \sqrt{T\sum_{t=1}^T \left( \left\|\widetilde{\phi_t}(\vx_t) - \phi^\star(\vx_t)\right\|\left\| \widetilde{\vbeta}_{a_t} \right\|  + {\left\| \begin{bmatrix}
        \vx_t \\
        \phi^\star(\vx_t)
    \end{bmatrix} \right\|_{\mA_{t-1, a_t}^{-1}}}{\left\|\begin{bmatrix}
        \widetilde{\vtheta}_{a_t} - \vtheta_{a_t}  \\
        \widetilde{\vbeta}_{a_t} - \vbeta_{a_t}
    \end{bmatrix}\right\|_{\mA_{t-1, a_t}}} \right)^2}  \\
    &\leq \sqrt{T\left(\sum_{t=1}^T  2\left\|\widetilde{\phi_t}(\vx_t) - \phi^\star(\vx_t)\right\|^2\left\| \widetilde{\vbeta}_{a_t} \right\|^2  + 2\zeta_{T,a}^2\left(1\land{\left\| \begin{bmatrix}
        \vx_t \\
        \phi^\star(\vx_t)
    \end{bmatrix} \right\|_{\mA_{t-1, a_t}^{-1}}^2}  \right)  \right) }   \\
    &\leq \sqrt{T\left(\sum_{t=1}^T  2\left\|\widetilde{\phi_t}(\vx_t) - \phi^\star(\vx_t)\right\|^2\left\| \widetilde{\vbeta}_{a_t} \right\|^2  + 2\zeta_{T,a}^2\left(1\land{\left\| \begin{bmatrix}
        \vx_t \\
        \phi^\star(\vx_t)
    \end{bmatrix} \right\|_{\mA_{t-1, a_t}^{-1}}^2}  \right)  \right) }  \\
    &\leq 
    \sqrt{T\cdot 
    \left(
        8C_0T^{1- 2\alpha}
        \log^{ 2\alpha}\left(\frac{\det \mX_t}{\det \mX_0}\right)\log^2\left(\frac{T}{\delta}\right) 
        + 2K\zeta_T^2 d_u \log\left(1+\frac{T L_u^2}{\lambda d_u}\right) 
    \right)}
\end{align*}

In the next, we expand the term $\zeta_{T}^2$, 
\begin{align*}
    \zeta_{T}^2 &\leq  \left(\sqrt{2(L_{\epsilon}^2+R_{\eta}^2) \log\left(\frac{1+TL_{u}^2/\lambda}{\delta/2}\right)} +  \frac{{L_u} \left(\sum_{s=1}^T e^{\delta/T}_s\right)}{\sqrt{\lambda}} + 2\sqrt{\lambda}\right)^2 \\
    &\leq 6(L_{\epsilon}^2+R_{\eta}^2) \log\left(\frac{1+TL_{u}^2/\lambda}{\delta/2}\right) + \frac{3L_u^2}{\lambda}\left(\sum_{s=1}^T e^{\delta/T}_s\right)^2 + 12\lambda 
\end{align*}

In the next, we bound the second term in the above equation under the learnability assumption~\ref{assump:learnabilityOnPhiAdvNew},
\begin{align*}
    \left(\sum_{s=1}^T e^{\delta/T}_s\right)^2 \leq 16T^{2-2\alpha}\log^{2\alpha}\left(\frac{\det \mX_T}{\det \mX_0}\right)\log^2\left(\frac{T}{\delta}\right).
\end{align*}

Therefore, naively choosing the value of $\lambda$ will lead to a linear regret due to the term $T^{3/2-\alpha}$ in the equation. To minimize the upper bound, we can choose the value of $\lambda$ to be
\begin{align*}
    \lambda = 2L_u T^{1-\alpha} \log^\alpha\left( \frac{\det \mX_T}{\det \mX_0} \right)\log\left(\frac{T}{\delta}\right).
\end{align*}
Then, we can bound $\zeta_T^2$ by
\begin{align*}
    \zeta_T^2 \leq   6(L_{\epsilon}^2+R_{\eta}^2) \log\left(\frac{1+TL_{u}^2/\lambda}{\delta/2}\right) + 48L_u T^{1-\alpha} \log^\alpha\left( \frac{\det \mX_T}{\det \mX_0} \right)\log\left(\frac{T}{\delta}\right).
\end{align*}
By plugging it in and following the simplication as in the proof of Theorem~\ref{thm:regret-pLinUCB}, we can get the regret is upper bounded by
\begin{align*}
   \widetilde{\mathcal{O}}\left(  T^{1-\alpha}d_u^\alpha + T^{1-\alpha/2}\sqrt{K d_u^{1+\alpha}} \right).
\end{align*}
The above result is summarized as the following proposition.
\thmRegretPLINUCBPhi*

\section{Missing Proofs}\label{appendix:proofs}

\subsection{Missing Proofs in the Generalized Elliptical Potential Lemma}\label{appendix:gepl}

\lemGEPL*

\begin{proof}
Our proof follows the high level idea for proving the  original EPL. However, to accommodate the noises in the data matrix, we have to introduce new matrix concentration tools to the original (primarily algebraic) proof,  and also identify the right conditions for the argument to go through. A key lemma to our proof is the following high probability bound regarding the noisy data matrix:   

\begin{restatable}[]{lem}{lemRobustElliptical}\label{lem:noisy-matrix-psd}
Let $\vx_1, ..., \vx_T \in \R^d$ be a fixed sequence of vectors, and $\vepsilon_1, ..., \vepsilon_T \in \R^d$ are independent random variables satisfying $\max_{t} \|\vx_t\|_2 \leq L_x$,  $\max_t\|\vepsilon_t\|_2 \leq L_{\epsilon}$, and $\expect[\vepsilon_t \vepsilon_t^\top] \succcurlyeq \sigma_\epsilon^2 \mI$. %
Then the following hold with probability at least $ 1- 2d \exp\left(\frac{-T\sigma_{\epsilon}^4}{8L_{\epsilon}^2 (L_{\epsilon} + L_x)^2}\right)$, 
\begin{align*}
    \sum_{t=1}^T (\vx_t + \vepsilon_t)(\vx_t + \vepsilon_t)^\top \succcurlyeq \sum_{t=1}^T  \vx_t \vx_t^\top.
\end{align*}
\end{restatable}

{ The Proof of Lemma~\ref{lem:noisy-matrix-psd} employs the Bernstein's Inequality for matrices \citep{tropp2015introduction}, which is technical; for   ease of presentation, we defer its proof of Appendix \ref{appendix:gepl}. By Lemma~\ref{lem:noisy-matrix-psd},  we have that, for every  $t \in [T]$,  the following inequality holds   with  probability at least  $1- 2d \exp\left(\frac{-t\sigma_{\epsilon}^4}{8L_{\epsilon}^2 (L_{\epsilon} + L_x)^2}\right)$:  }
\begin{align*}
    \widetilde{\mX}_t = \mX_0 + \sum_{s=1}^t (\vx_s + \vepsilon_s)(\vx_s + \vepsilon_s)^\top \succcurlyeq \mX_0 + \sum_{s=1}^t \vx_s\vx_s^\top \coloneqq {\mX}_t,
\end{align*}
under which we have $$\|\vx_{t+1}\|_{\widetilde{\mX}^{-1}_{t}}^2  \leq \|\vx_{t+1}\|_{{\mX}^{-1}_{t}}^2.$$ 

To prove our Lemma \ref{lem:gepl}, we need to apply union bound to guarantee the above hold simultaneously for every $t \geq 1$ with high probability. Unfortunately, this turns out to not be true   because when $t$ is very small (e.g., $t=1$), the above inequality cannot hold with high probability. Therefore, to obtain high-probability guarantee by the union bound, we will have to exclude these small $t$'s and apply the union bound for only the   events from some $t' \in [T]$, as follows 
\begin{align*}
    &\prob\left[ \forall t \in [t', T],\;\;  \|\vx_t\|_{\widetilde{\mX}^{-1}_{t-1}}^2  \leq \|\vx_t\|_{{\mX}^{-1}_{t-1}}^2 \right] \\
    &\geq 1 - \sum_{t=t'-1}^{T-1} 2d\exp\left(\frac{-t \sigma_\epsilon^4}{(8 L_\epsilon^2 (L_\epsilon + L_x)^2)}\right) \\
    &\geq 1 - \sum_{t=t'-1}^{\infty} 2d\exp\left(\frac{-t \sigma_\epsilon^4}{(8 L_\epsilon^2 (L_\epsilon + L_x)^2)}\right) \\
    & =  1 - 2d\left( \frac{\exp\big(-(t'-1) \sigma_\epsilon^4/(8 L_\epsilon^2 (L_\epsilon + L_x)^2) \big)}{1-\exp\big( -\sigma_\epsilon^4/(8 L_\epsilon^2 (L_\epsilon + L_x)^2) \big)} \right). \\
    & \geq 1 - 2d \times   \exp \big(-(t'-1) \sigma_\epsilon^4/(8 L_\epsilon^2 (L_\epsilon + L_x)^2) \big)  \times \frac{ 16 L_\epsilon^2 (L_\epsilon + L_x)^2}{ \sigma_\epsilon^4},  
\end{align*}
where the last inequality uses the fact that $\sigma_\epsilon^4/ (L_\epsilon^2 (L_\epsilon + L_x)^2) \leq (\sigma_\epsilon/ L_\epsilon)^4 \leq 1$ and  $1 - e^{-x} \geq x/2$ for any $x \in [0, 1]$.  By solving the following inequality,
\begin{align*}
 \exp \big(-(t'-1) \sigma_\epsilon^4/(8 L_\epsilon^2 (L_\epsilon + L_x)^2) \big)  \times \frac{ 16 L_\epsilon^2 (L_\epsilon + L_x)^2}{ \sigma_\epsilon^4} \leq \frac{\delta}{2d}, 
\end{align*}
we have,
\begin{align*}
   t' & \geq 1 +  \frac{  8 L_\epsilon^2 (L_\epsilon + L_x)^2 }{\sigma_{\epsilon}^4} \log\left(\frac{32d L_\epsilon^2 (L_\epsilon + L_x)^2 }{\delta \sigma_{\epsilon}^4}\right) 
\end{align*}
Let $T_0$ denote the ceiling of the right-hand-side of the above term. 
Therefore, we have the following hold with high probability at least $1 - \delta$: 
\begin{align*}
     \sum_{t=1}^T \left(1 \land \|\vx_t\|_{\widetilde{\mX}^{-1}_{t-1}}^2\right)^p &\leq  (T_0 - 1) + \sum_{t=T_0}^T \left(1 \land \|\vx_t\|_{\widetilde{\mX}^{-1}_{t-1}}^2\right)^p\\
     &\leq (T_0 - 1) + \sum_{t=T_0}^T \left(1 \land \|\vx_t\|_{{\mX}^{-1}_{t-1}}^2\right)^p \\
     &\leq (T_0 - 1) + \sum_{t=T_0}^T \left(1 \land \|\vx_t\|_{{\mX}^{-1}_{t-1}}^2\right)^p 
\end{align*}

In the next, we are going to bound the second term in the above equation, whose proof can be adapted from the proof of the original elliptical potential lemma. Using the fact that for any $z \in [0, +\infty]$, $z \land 1 \leq 2\ln(1+z)$, we get
\begin{align*}
    \sum_{t=1}^T 1 \land \left(\|\vx_t\|_{{\mX}_{t-1}^{-1}}^2\right)^{p} \leq \sum_{t=1}^T \left(2 \log\left( 1 + \|\vx_t\|_{{\mX}_{t-1}^{-1}}^2 \right)\right)^{p}.
\end{align*}
Additionally, by definition, we have
\begin{align*}
    {\mX}_{t} = {\mX}_{t-1} + \vx_{t}\vx_t^\top = \mX_{t-1}^{1/2}\left(\mI + \mX_{t-1}^{-1/2}\vx_t \vx_t^\top \mX_{t-1}^{-1/2}\right)\mX_{t-1}^{1/2}.
\end{align*}
This implies the following relationship between the determinant,
\begin{align*}
    \det \mX_{t} = \det\left(\mX_{t-1}\right)\det\left(\mI + \mX_{t-1}^{-1/2}\vx_t\vx_t^\top \mX_{t-1}^{-1/2}\right).
\end{align*}
Since the only eigenvalues of a matrix of the form $\mI + \vy\vy^\top$ are $1+\|\vy\|_2$ and $1$, we have
\begin{align*}
     \log\left( 1 + \|\vx_t\|_{\mX_{t-1}^{-1}}^2 \right)= \log\det\mX_{t} - \log\det\mX_{t-1}.
\end{align*}
By taking the power $p$ for both sides and taking the sum, we have
\begin{align*}
     \sum_{t=1}^T \left(\log\left( 1 + \|\vx_t\|_{\mX_{t-1}^{-1}}^2 \right)\right)^{p}= \sum_{t=1}^T\left( \log\det \mX_t - \log\det \mX_{t-1}\right)^{p}.
\end{align*}
Since $p \in [0, 1]$, the function $g(x) = x^{p}$ is a concave function. Thus, we have
\begin{align*}
    \frac{1}{T}\sum_{t=1}^T\left( \log\det \mX_t - \log\det \mX_{t-1}\right)^{p} &\leq \left(\frac{1}{T}\sum_{t=1}^T \log\det \mX_t - \log\det \mX_{t-1}\right)^{p} =\frac{1}{T^{p}}\log^{p}\left(\frac{\det\mX_T}{\det\mX_0}\right).
\end{align*}
Therefore, we can conclude that
\begin{align*}
     \sum_{t=1}^T 1 \land \left(\|\vx_t\|^2_{\mX_{t-1}^{-1}}\right)^p \leq 2^p T^{1-p}\log^{p}\left(\frac{\det\mX_T}{\det\mX_0}\right).
\end{align*}
By combining the above results, we have the following hold with probability at least $1-\delta$: 
\begin{align*}
    \sum_{t=1}^T \left(1 \land \|\vx_t\|_{\widetilde{\mX}^{-1}_{t-1}}^2\right)^p \leq T_0 - 1  +  \left(\sum_{t=1}^T 1 \land \|\vx_t\|_{{\mX}^{-1}_{t-1}}^2\right)^p
    \leq T_0 - 1 + 2^p T^{1-p}\log^{p}\left(\frac{\det\mX_T}{\det\mX_0}\right).
\end{align*}

Invoking $$T_0 - 1 \leq \frac{  8 L_\epsilon^2 (L_\epsilon + L_x)^2 }{\sigma_{\epsilon}^4} \log\left(\frac{32d L_\epsilon^2 (L_\epsilon + L_x)^2 }{\delta \sigma_{\epsilon}^4}\right)  ,$$ we obtained the desired inequality with probability at least $1-\delta$: 
\begin{align*}
   \sum_{t=1}^T \left(1 \land \|\vx_t\|_{\widetilde{\mX}^{-1}_{t-1}}^2\right)^p \leq 2^p T^{1-p}\log^p\left(\frac{\det \mX_T}{\det\mX_0}\right) +  \frac{  8 L_\epsilon^2 (L_\epsilon + L_x)^2 }{\sigma_{\epsilon}^4} \log\left(\frac{32d L_\epsilon^2 (L_\epsilon + L_x)^2 }{\delta \sigma_{\epsilon}^4}\right)  .
\end{align*}

\end{proof}

\lemRobustElliptical*
\begin{proof}
 We will analyze the term-wise difference, denoted as 
    \begin{align}
        \mS_t \coloneqq (\vx_t + \vepsilon_t)(\vx_t + \vepsilon_t)^\top  -  \vx_t \vx_t^\top = \vepsilon_t \vx_t^\top +\vx_t\vepsilon_t^\top + \vepsilon_t\vepsilon_t^\top
    \end{align}
Moreover, since  $ \expect[\vepsilon_t] = 0 $ for any $t \in [T]$, the expectation of $\mS_t$ (over randomness of noise) can be lower bounded as
    \begin{align*}
        \expect[\mS_t] = \expect[\vepsilon_t] \vx_t^\top +\vx_t \expect[\vepsilon_t^\top] + \expect[ \vepsilon_t\vepsilon_t^\top] =  \expect[\vepsilon_t \vepsilon_t^\top] \succcurlyeq \sigma_\epsilon^2 \mI.
    \end{align*}
    Since $\|\vx_t\|_2 \leq L_x$ and $\|\vepsilon_t\|_2 \leq L_{\epsilon}$, we know that $ \mS_t \preccurlyeq (2L_{\epsilon} L_x + L_{\epsilon}^2) \mI$ with probability $1$.  Thus $\mS_t$ is uniformly upper bounded under the spectral-norm denoted by $\|  \cdot \|$, or formally 
    \begin{align*}
       \|  \mS_t \|  \leq  2L_{\epsilon} L_x + L_{\epsilon}^2   .
    \end{align*}
Consider the ``centered'' matrix sum $ \mZ_T = \sum_{t=1}^T \bigg[ \mS_t - \expect[\vepsilon_t\vepsilon_t^\top] \bigg]$, with mean $0$. {Since the spectral-norm of $\mS_t$ is upper bounded by $2L_{\epsilon} L_x + L_{\epsilon}^2$, its variance $ \sV(\mS_t) = \left\|\expect\left[\mS_t\mS_t^\top\right]\right\|_2$  is   upper bounded by $(2L_{\epsilon} L_x + L_{\epsilon}^2)^2$. Thus, the variance of $\mZ_T$ equals  the variance of sum $\sum_{t=1}^T \mS_t$, which is then upper bounded by $T(2L_{\epsilon} L_x + L_{\epsilon}^2)^2$. 
}  By the Bernstein's Inequality for random matrices~\citep{tropp2015introduction}, we have the following high probability upper bound for the spectral norm $\| \cdot \|$ of $\mZ_T$: 
\begin{align}\label{eqn:matrix-bernstein-iota}
    \mathbb{P}\left[\|\mZ_T\|  \geq  \iota \right] \leq 2d \exp\left(\frac{-\iota^2/2}{\sV(\mZ_T) + (2L_{\epsilon} L_x + L_{\epsilon}^2)\iota/3}\right).
\end{align}
Since $\mZ_T$ is a symmetric matrix, its spectral norm upper bounds the  absolute value of any eigenvalue. Thus we can lower bound the smallest eigenvalues as follows: 
\begin{align*}
    \sP\left[\|\mZ_T\|  \geq  \iota \right] &\geq  \sP\left[ \lambda_{\min}\left(\sum_{t=1}^T \mS_t - \expect[\vepsilon_t\vepsilon_t^\top]\right) \leq - \iota \right]\\
    &\geq  \sP\left[ \lambda_{\min}\left(\sum_{t=1}^T \mS_t - \sigma_\epsilon^2 \mI \right) \leq - \iota \right]\\
    & = \sP\left[ \lambda_{\min}\left(\sum_{t=1}^T \mS_t\right) \leq T\sigma^2_{\epsilon} - \iota \right].
\end{align*}
where the second inequality is due to two facts: (1) $A \succcurlyeq B$ implies $\lambda_{\min}(A) \geq \lambda_{\min} (B)$ where $A = \left(\sum_{t=1}^T \mS_t - \sigma_\epsilon^2 \mI \right)$ and $B = \left(\sum_{t=1}^T \mS_t - \expect[\vepsilon_t\vepsilon_t^\top] \right) $; and (2) the event $\lambda_{\min} (A) \leq - \iota $ thus is included within the event $\lambda_{\min} (B) \leq - \iota $.  Consequently, we have 
\begin{align*}
    \sP\left[ \lambda_{\min}\left(\sum_{t=1}^T \mS_t\right) \leq T\sigma^2_{\epsilon} - \iota \right] \leq 2d \exp\left(\frac{-\iota^2/2}{\sV(\mZ_T) + (2L_{\epsilon} L_x + L_{\epsilon}^2)\iota/3}\right),
\end{align*}
or equivalently,
\begin{align*}
    \sP\left[ \lambda_{\min}\left(\sum_{t=1}^T \mS_t\right) \geq T\sigma^2_{\epsilon} - \iota \right] \geq 1- 2d \exp\left(\frac{-\iota^2/2}{\sV(\mZ_T) + (2L_{\epsilon} L_x + L_{\epsilon}^2)\iota/3}\right),
\end{align*}
By choosing the value of $\iota = T\sigma^2_{\epsilon}$, we get
\begin{align*}
     &\sP\left[ \lambda_{\min}\left(\sum_{t=1}^T \mS_t\right) \geq 0 \right] \\
     &\geq 1- 2d \exp\left(\frac{-T^2\sigma_{\epsilon}^4/2}{\sV(\mZ_T) + (2L_{\epsilon} L_x + L_{\epsilon}^2)T\sigma^2_{\epsilon}/3}\right)\\
     &\geq  1- 2d \exp\left(\frac{-T^2\sigma_{\epsilon}^4/2}{T(2L_{\epsilon} L_x + L_{\epsilon}^2)^2 + (2L_{\epsilon} L_x + L_{\epsilon}^2)T\sigma^2_{\epsilon}/3}\right) . 
\end{align*}
Using the fact that $\sigma_\epsilon \leq L_\epsilon$,  we can further simplify the above equation by 
\begin{align*}
\sP\left[ \lambda_{\min}\left(\sum_{t=1}^T \mS_t\right) \geq 0 \right] \geq  1- 2d \exp\left(\frac{-T\sigma_{\epsilon}^4}{8L_{\epsilon}^2 (L_{\epsilon} + L_x)^2}\right) . 
\end{align*}
\end{proof}

\subsection{Missing Proofs in the Regret Analysis}\label{appendix:regret}

\thmRegretPLINUCB*

\begin{proof}
    In the next, we prove the regret bound. For each time step $t$, the immediate regret is
\begin{align*}
    \Delta_t &=r_{t, a^\star_t} - r_{t,a_t} \\
    & = \left\langle \begin{bmatrix}
        \vx_t \\
        \phi^\star(\vx_t)
    \end{bmatrix}, \begin{bmatrix}
        \vtheta_{a^\star_t} - \vtheta_{a_t}^\star  \\
        \vbeta_{a^\star_t} - \vbeta_{a_t}^\star
    \end{bmatrix} \right\rangle  \\
    & \stackrel{\mathrm{(a)}}{\leq} \left\langle \begin{bmatrix}
        \vx_t \\
        \widetilde{\phi_t}(\vx_t)
    \end{bmatrix}, \begin{bmatrix}
        \widetilde{\vtheta}_{a_t} - \vtheta_{a_t}^\star  \\
        \widetilde{\vbeta}_{a_t} - \vbeta_{a_t}^\star
    \end{bmatrix} \right\rangle  + \left\langle \widetilde{\phi}_t(\vx_t)-\phi^\star(\vx_t),  \vbeta_{a_t}^\star\right\rangle \\
    & = \left\langle \begin{bmatrix}
        \mathbf{0} \\
        \widetilde{\phi_t}(\vx_t)-\phi^\star(\vx_t)
    \end{bmatrix} +  \begin{bmatrix}
        \vx_t \\
        \phi^\star(\vx_t)
    \end{bmatrix}, \begin{bmatrix}
        \widetilde{\vtheta}_{a_t} - \vtheta_{a_t}^\star  \\
        \widetilde{\vbeta}_{a_t} - \vbeta_{a_t}^\star
    \end{bmatrix} \right\rangle   + \left\langle \widetilde{\phi}_t(\vx_t)-\phi^\star(\vx_t),  \vbeta_{a_t}^\star\right\rangle\\
    & = \left\langle 
        \widetilde{\phi_t}(\vx_t)-\phi^\star(\vx_t)
 ,  \widetilde{\vbeta}_{a_t} \right\rangle +  \left\langle\begin{bmatrix}
        \vx_t \\
        \phi^\star(\vx_t)
    \end{bmatrix}, \begin{bmatrix}
        \widetilde{\vtheta}_{a_t} - \vtheta_{a_t}^\star \\
        \widetilde{\vbeta}_{a_t} - \vbeta_{a_t}^\star
    \end{bmatrix} \right\rangle  \\
    &\stackrel{\mathrm{(b)}}{\leq} \left\|\widetilde{\phi_t}(\vx_t) - \phi^\star(\vx_t)\right\|\cdot\left\| \widetilde{\vbeta}_{a_t} \right\|  + \left\langle \begin{bmatrix}
        \vx_t \\
        \phi^\star(\vx_t)
    \end{bmatrix}, \begin{bmatrix}
        \widetilde{\vtheta}_{a_t} - \vtheta_{a_t}^\star  \\
        \widetilde{\vbeta}_{a_t} - \vbeta_{a_t}^\star
    \end{bmatrix} \right\rangle  \\
    &\stackrel{\mathrm{(c)}}{\leq}  \left\|\widetilde{\phi_t}(\vx_t) - \phi^\star(\vx_t)\right\|\cdot\left\| \widetilde{\vbeta}_{a_t} \right\|  + {\left\| \begin{bmatrix}
        \vx_t \\
        \phi^\star(\vx_t)
    \end{bmatrix} \right\|_{\mA_{t-1, a_t}^{-1}}}{\left\|\begin{bmatrix}
        \widetilde{\vtheta}_{a_t} - \vtheta_{a_t}^\star  \\
        \widetilde{\vbeta}_{a_t} - \vbeta_{a_t}^\star
    \end{bmatrix}\right\|_{\mA_{t-1, a_t}}}.
\end{align*}

where the inequality (a) is due to the definition of UCB, and (b) and (c) are obtained using the Cauchy-Schwarz inequality. Therefore, the cumulative regret can be further upper bounded by 
\begin{align*}
    &R_T =\sum_{t=1}^T \Delta_t \leq \sqrt{T\sum_{t=1}^T \Delta_t^2}\\
    &\leq \sqrt{T\sum_{t=1}^T \left( \left\|\widetilde{\phi_t}(\vx_t) - \phi^\star(\vx_t)\right\|\left\| \widetilde{\vbeta}_{a_t}\right\|  + {\left\| \begin{bmatrix}
        \vx_t \\
        \phi^\star(\vx_t)
    \end{bmatrix} \right\|_{\mA_{t-1, a_t}^{-1}}}{\left\|\begin{bmatrix}
        \widetilde{\vtheta}_{a_t} - \vtheta_{a_t}  \\
        \widetilde{\vbeta}_{a_t} - \vbeta_{a_t}
    \end{bmatrix}\right\|_{\mA_{t-1, a_t}}} \right)^2}  \\
    &\leq \sqrt{T\left(\sum_{t=1}^T  2\left\|\widetilde{\phi_t}(\vx_t) - \phi^\star(\vx_t)\right\|^2\left\| \widetilde{\vbeta}_{a_t}\right\|^2  + 2\zeta_T^2\left(1\land{\left\| \begin{bmatrix}
        \vx_t \\
        \phi^\star(\vx_t)
    \end{bmatrix} \right\|_{\mA_{t-1, a_t}^{-1}}^2}  \right)  \right) }   \\
    &\leq \sqrt{T\left(\sum_{t=1}^T  2\left\|\widetilde{\phi_t}(\vx_t) - \phi^\star(\vx_t)\right\|^2\left\| \widetilde{\vbeta}_{a_t} \right\|^2  + 2\zeta_T^2\left(1\land{\left\| \begin{bmatrix}
        \vx_t \\
        \phi^\star(\vx_t)
    \end{bmatrix} \right\|_{\mA_{t-1, a_t}^{-1}}^2}  \right)  \right) }  
\end{align*}
In the next, we bound each term separately. Firstly, we have the following hold with probability at least $1-\delta$ by using the union bound,
\begin{align*}
    \sum_{t=1}^T  2\left\|\widetilde{\phi_t}(\vx_t) - \phi^\star(\vx_t)\right\|^2\cdot\left\| \widetilde{\vbeta}_{a_t} \right\|^2 \leq 8\sum_{t=1}^T \left(e^{\delta/T}_t\right)^2.
\end{align*}
In the next, to bound the remaining term, we use the result from Lemma~\ref{lem:gepl}. We first group the sums based the arm,
\begin{align}\label{eqn:epl-bound}
    \sum_{t=1}^T  1\land{\left\| 
    \begin{bmatrix}
        \vx_t \\
        \phi^\star(\vx_t)
    \end{bmatrix} 
    \right\|_{\mA_{t-1, a_t}^{-1}}^2} &= \sum_{a\in \calA} \sum_{t\in[T]: a_t = a}  1\land{\left\| \begin{bmatrix}
        \vx_t \\
        \phi^\star(\vx_t)
    \end{bmatrix} \right\|_{\mA_{t-1, a}^{-1}}^2}
\end{align}
By denoting $n_T(a)$ as the number of times that arm $a$ is pulled, we  can divide the arms into two groups,
\begin{align*}
    \calG_{0} \coloneqq \left\{a\in \calA: n_T(a) = \Omega(\log(1/\delta)) \right\} \quad \textnormal{and} \quad \calG_{1} \coloneqq \calA \setminus \calG_{0}.
\end{align*}

Then, we can further decompose the r.h.s term of Equation~\ref{eqn:epl-bound} based on if the corresponding arm is in $\mathcal{G}_0$ or $\mathcal{G}_1$. Then, by applying Lemma~\ref{lem:gepl}, we have the following holds, with probability at least $1-\delta$,
\begin{align*}
   & \sum_{a\in \calA} \sum_{t\in[T]: a_t = a}  1\land{\left\| \begin{bmatrix}
        \vx_t \\
        \phi^\star(\vx_t)
    \end{bmatrix} \right\|_{\mA_{t-1, a}^{-1}}^2} \\
    =& \sum_{a \in \calG_0 } \sum_{t\in[T]: a_t = a}  1\land{\left\| \begin{bmatrix}
        \vx_t \\
        \phi^\star(\vx_t)
    \end{bmatrix} \right\|_{\mA_{t-1, a}^{-1}}^2} + \sum_{a \in \calG_1 } \sum_{t\in[T]: a_t = a}  1\land{\left\| \begin{bmatrix}
        \vx_t \\
        \phi^\star(\vx_t)
    \end{bmatrix} \right\|_{\mA_{t-1, a}^{-1}}^2} \\
    \leq& {2Kd_u\log\left(1 + \frac{TL_u^2}{\lambda d_u}\right) +  \frac{  8 KL_\epsilon^2 (L_\epsilon + L_x)^2 }{\sigma_{\epsilon}^4} \log\left(\frac{32Kd_u L_\epsilon^2 (L_\epsilon + L_x)^2 }{\delta \sigma_{\epsilon}^4}\right) } .
\end{align*}
where the last inequality is due to Lemma~\ref{lem:gepl} and apply the union bound on the $K$ arms. To bound the remainder term, since $\alpha \in [0, 1/2]$, by the learnability assumption as stated in Assumption~\ref{assump:learnabilityOnPhiAdvNew} and Lemma~\ref{lem:gepl}, we have
\begin{align*}
    \sum_{t=1}^T \left(e^{\delta/t}_t\right)^2 &\leq \sum_{t=1}^T C_0^2 \cdot \left(1 \land \|\vx\|^2_{{\mX^{-1}_{t-1}}}\right)^{2\alpha} \cdot \log^2\left(\frac{tT}{\delta}\right) \\
    &\leq 4C_0^2T^{1- 2\alpha}\log^{ 2\alpha}\left(\frac{\det\mX_T}{\det\mX_0}\right)\log^2\left(\frac{T}{\delta}\right) \\
    &\leq 4C_0^2T^{1- 2\alpha}d_u^{2\alpha}\log^{ 2\alpha}\left(\frac{TL_u^2/d+\lambda}{\lambda}\right)\log^2\left(\frac{T}{\delta}\right) 
\end{align*}

Therefore, the total regret bound is bounded by the following term with probability at least $1-2\delta$,
{\small
\begin{align*}
    \sqrt{T\cdot 
    \left(
        8C_0T^{1- 2\alpha}
        \log^{ 2\alpha}\left(\frac{\det \mX_T}{\det \mX_0}\right)\log^2\left(\frac{T}{\delta}\right) 
        + K\zeta_T^2\left( d_u \log\left(1+\frac{T L_u^2}{\lambda d_u}\right) 
    + \frac{48 L_{\epsilon}^4L_u^2}{\sigma_{\epsilon}^4} \log\left(\frac{192Kd_uL_\epsilon^4 L_u^2}{\delta \sigma_\epsilon^4}\right)\right) 
    \right)}
\end{align*}}

By hiding the logarithmic terms, we can further simplify it to be
\begin{align*}
     \widetilde{\mathcal{O}}\left(T^{1- \alpha}d_u^{ \alpha} + d_u\sqrt{T K }\right)
\end{align*}
\end{proof}

\subsection{Missing Proofs in Generalizations}\label{appendix:extensitions}

\propActionDependentContextsRegret*

 \begin{proof}
     Our proof follows from the proof of Theorem~\ref{thm:regret-pLinUCB}. The immediate regret at each time step $t$ is
     \begin{align}
         \Delta_t = r_{t, a_t^\star} - r_{t, a_t}
     \end{align}
     Recall that the definition of $a^\star_t$,
     \begin{align}
         a_t^\star \coloneqq \argmax_{a \in \calA} \langle \vtheta_{a}^\star, \vx_{t, a} \rangle + \langle \vbeta_{a}^\star, \phi_{a}^\star(\vx_{t, a}) \rangle.
     \end{align}
     To bound the immediate regret, we have
     \begin{align}
           \Delta_t = \langle\vtheta_{a^\star_t}^\star, \vx_{t, a^\star_t} \rangle + \langle \vbeta_{a^\star_t}^\star, \phi_{a^\star_t}^\star(\vx_{t, a^\star_t}) \rangle - \langle\vtheta_{a_t}^\star, \vx_{t, a_t} \rangle - \langle \vbeta_{a_t}^\star, \phi_{a_t}^\star(\vx_{t, a_t}) \rangle.
     \end{align}
     By the definition of UCB, we further have
     \begin{align}
         \Delta_t \leq  \langle \widetilde{\vtheta}_{t, a_t}, \vx_{t, a_t} \rangle + \langle \widetilde{\vbeta}_{t, a_t}, \widetilde{\phi}_{t, a_t}^\star(\vx_{t, a_t}) \rangle - \langle\vtheta_{a_t}^\star, \vx_{t, a_t} \rangle - \langle \vbeta_{a_t}^\star, \phi_{a_t}^\star(\vx_{t, a_t}) \rangle.
     \end{align}
     By rearranging the terms, we get
     \begin{align}
         \Delta_t \leq \underbrace{\langle \widetilde{\vtheta}_{t, a_t} - \vtheta_{a_t}^\star, \vx_{t, a_t} \rangle + \langle \widetilde{\vbeta}_{t, a_t} - \vbeta_{a_t}^\star, {\phi}_{a_t}^\star(\vx_{t, a_t}) \rangle}_{\circled{1}} + \underbrace{\langle \widetilde{\vbeta}_{t, a_t}, \widetilde{\phi}_{t, a_t}(\vx_{t, a_t}) - \phi_{a_t}^\star(\vx_{t, a_t}) \rangle}_{\circled{2}} 
     \end{align}
     Bounding the first term $\circled{1}$ is the same as the proof in Theorem~\ref{thm:regret-pLinUCB}, while bounding the second term $\circled{2}$ will be slightly different, as we now have $K$ such functions of $\vphi_{a}^\star(\cdot)$ for $a \in \calA$ to learn. By denoting the error for each estimate of ${\phi}^{\star}_{a}(\cdot)$ at iteration $t$ as $e^{\delta}_{t, a}$. Therefore, the contribution from the second term to the total regret can be bounded by
     \begin{align}
         \sum_{t=1}^T \left(e^{\delta/t}_{t, a_t}\right)^2 &= \sum_{a \in \calA} \sum_{t\in [T]: a_t = a} \left(e^{t/\delta}_{t, a}\right)^2 \\
         &\leq 4KC_0T^{1-2\alpha}\log^{2\alpha}\left(\frac{\det\mX_T}{\det\mX_0}\right)\log^2\left(\frac{T}{\delta}\right). 
     \end{align}
     Hence, by following the remaining steps in the proof of Theorem~\ref{thm:regret-pLinUCB}, we can conclude that the regret is upper bounded by
     \begin{align}
          \widetilde{\mathcal{O}}\left(T^{1-\alpha}d_u^{\alpha}\sqrt{K} + d_u\sqrt{T K }\right),
     \end{align}
     where the only difference is the additional $\sqrt{K}$ appeared in the first term.
 \end{proof}

\propAbbasiRegret*

 \begin{proof}
     This proof also follows from the proof of Theorem~\ref{thm:regret-pLinUCB}. The immediate regret at each time step $t$ is
     \begin{align}
         \Delta_t = r_{t, \vx_t^\star} - r_{t, \vx_t}
     \end{align}
     Recall that the definition of $\vx^\star_t$,
     \begin{align}
         \vx_t^\star \coloneqq \argmax_{\vx \in D_t} \langle \vtheta^\star, \vx \rangle + \langle \vbeta^\star, \phi^\star(\vx) \rangle.
     \end{align}     
     To bound the immediate regret, we have
     \begin{align}
           \Delta_t = \langle \vtheta^\star, \vx^\star_t \rangle + \langle \vbeta^\star, \phi^\star(\vx_t^\star) \rangle - \langle \vtheta^\star, \vx_t \rangle - \langle \vbeta^\star, \phi^\star(\vx_t) \rangle
     \end{align}
     By the definition of UCB, we further have
     \begin{align}
         \Delta_t \leq  \langle \widetilde{\vtheta}_{t}, \vx_{t} \rangle + \langle \widetilde{\vbeta}_{t}, \widetilde{\phi}_{t}^\star(\vx_{t}) \rangle - \langle\vtheta^\star, \vx_{t} \rangle - \langle \vbeta^\star, \phi^\star(\vx_{t}) \rangle.
     \end{align}
     By rearranging the terms, we get
     \begin{align}
         \Delta_t \leq \underbrace{\langle \widetilde{\vtheta}_{t} - \vtheta^\star, \vx_{t} \rangle + \langle \widetilde{\vbeta}_{t} - \vbeta^\star, {\phi}^\star(\vx_{t}) \rangle}_{\circled{1}} + \underbrace{\langle \widetilde{\vbeta}_{t}, \widetilde{\phi}_{t}(\vx_{t}) - \phi^\star(\vx_{t}) \rangle}_{\circled{2}} 
     \end{align}
     Since we only need to fit  a single $\vtheta^\star$, $\vbeta^\star$ and $\phi^\star(\cdot)$. We thus have the following bound for the total regret,
     \begin{align}
          \widetilde{\mathcal{O}}\left(T^{1-\alpha}d_u^{\alpha} + d_u\sqrt{T}\right).
     \end{align}
 \end{proof}

\section{Technical Lemmas}
\begin{restatable}[\textbf{Confidence Ellipsoid}, based on Theorem 2 of~\citep{abbasi2011improved}]{lem}{thm-conf-2}\label{thm:conf-2}
Let $\vw^{\star} \in \sR^d$, $\mV_0 = \lambda \mI $, $\lambda > 0$. For any $t \geq 0$, let $\vu_1, \cdots, \vu_t \in \sR ^{d}$, define $r_t = \innerproduct{\vu_t}{\vw ^{\star}} + \eta_t$ where $\eta_t$ is $R_{\eta}$-sub-Gaussian and assume that $\| \vw ^{\star} \|_2 \leq L_{\vw}$; let $\mV_t = \mV_0 + \sum_{s=1}^{t}\vu_s \vu_s^{\top}$ and $\hat{\vw}_t$ be the corresponding regularised least-square estimator. Then, for any $\delta > 0$ and $t \geq 0$, with probability at least $1 - \delta$, $\vw^{\star}$ lies in the set:
\begin{align}
    \calC_t=\left\{\vw \in \mathbb{R}^d:\left\|\hat{\vw}_t-\vw\right\|_{\mV_t} \leq \sqrt{\lambda} L_{\vw} + R_{\eta} \sqrt{2 \log \left(\frac{\operatorname{det}\left(\mV_t\right)^{1 / 2} \operatorname{det}(\lambda \mI)^{-1 / 2}}{\delta}\right)}\right\}.
\end{align}
Furthermore, if for all $t \geq 1$, $\| \vu_t \| \leq L_{\vu}$, then for any $\delta > 0$ and $t \geq 0$, with probability at least $1 - \delta$, $\vw^{\star}$ lies in the set:
\begin{align}
    \calC_t=\left\{\vw \in \mathbb{R}^d:\left\|\hat{\vw}_t-\vw\right\|_{\mV_t} \leq \sqrt{\lambda} L_{\vw} + R_{\eta} \sqrt{d \log \left(\frac{1 + tL_{\vu}^2 / \lambda}{\delta}\right)}\right\}.
\end{align}
\end{restatable}

\begin{restatable}[Bernstein's Inequality for Matrices, Theorem 6.1.1 of \citep{tropp2015introduction}]{lem}{lem-bern-matrix}\label{prop:ber-matrix} 
Let $\mX_1, \cdots, \mX_n \in \sR^{d_1 \times d_2}$ be independent and centered random matrices. Assume that for each $i \in [n]$, $\mX_i$ is \emph{uniformly bounded}, that is:
\begin{align}
    \sE \left[ \mX_i \right] = \vzero \text{\quad and \quad} \| \mX_i \| \leq B,
\end{align}
where $\| \cdot \|$ denotes the spectral-norm distance here. Introduce the sum 
\begin{align}
    \mZ = \sum_{i=1}^n \mX_i,
\end{align}
and let $\sV(\mZ)$ denote the matrix variance statistics of the sum $\mZ$:
\begin{align}
    \sV(\mZ) &= \max \left\{ \|\sE \left[ \mZ \mZ ^{*}\right]\|, \| \sE\left[ \mZ^{*} \mZ \right] \| \right\}  \\
    &= \max \left\{\left\| \sum_{i=1}^n \mX_i \mX_i^{*} \right\|, \left\| \sum_{i=1}^n \mX_i ^{*} \mX_i \right\| \right\},
\end{align}
where the asterisk $^{*}$ denotes the conjugate transpose operation. Then, for every $\epsilon \geq 0$, we have,
\begin{align}
    {\sP \Bigl( \| \mZ \| \geq \epsilon \Bigl) \leq (d_1 + d_2) \cdot \exp \left( \frac{-\epsilon^2 / 2}{\sV(\mZ) + B\epsilon/3} \right)}.
\end{align}
\end{restatable}

\newpage
\section{Additional Results and Discussions}

\begin{figure}[t]
    \centering
    \includegraphics[width=0.6\textwidth]{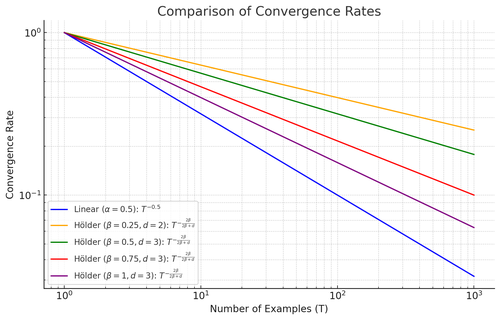}
    \caption{Comparison of convergence under different $\phi$ functions (e.g., linear and those in Holder space).}
    \label{fig:visual-illustraion-holder}
    \vspace{-0.3cm}
\end{figure}

\begin{figure}[t]
    \centering
    \includegraphics[width=0.475\textwidth]{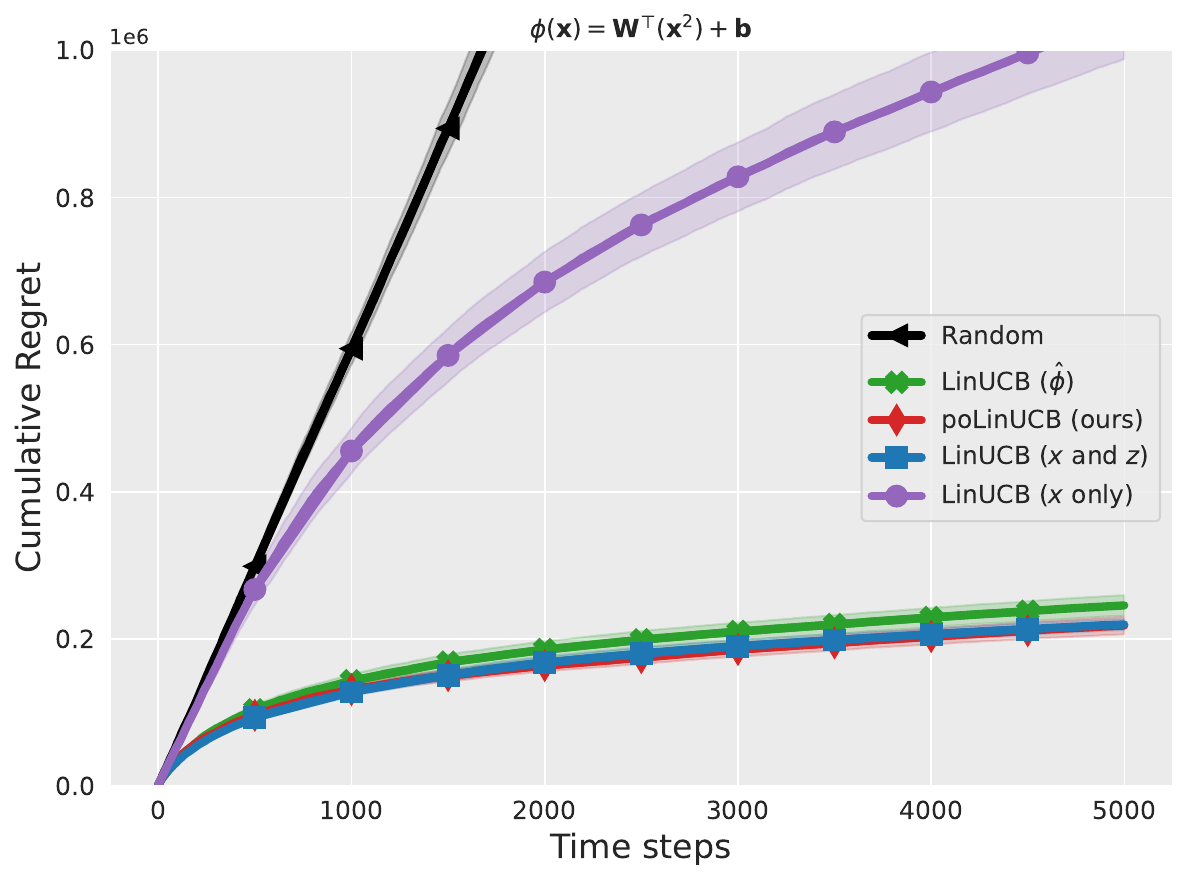}
    \includegraphics[width=0.475\textwidth]{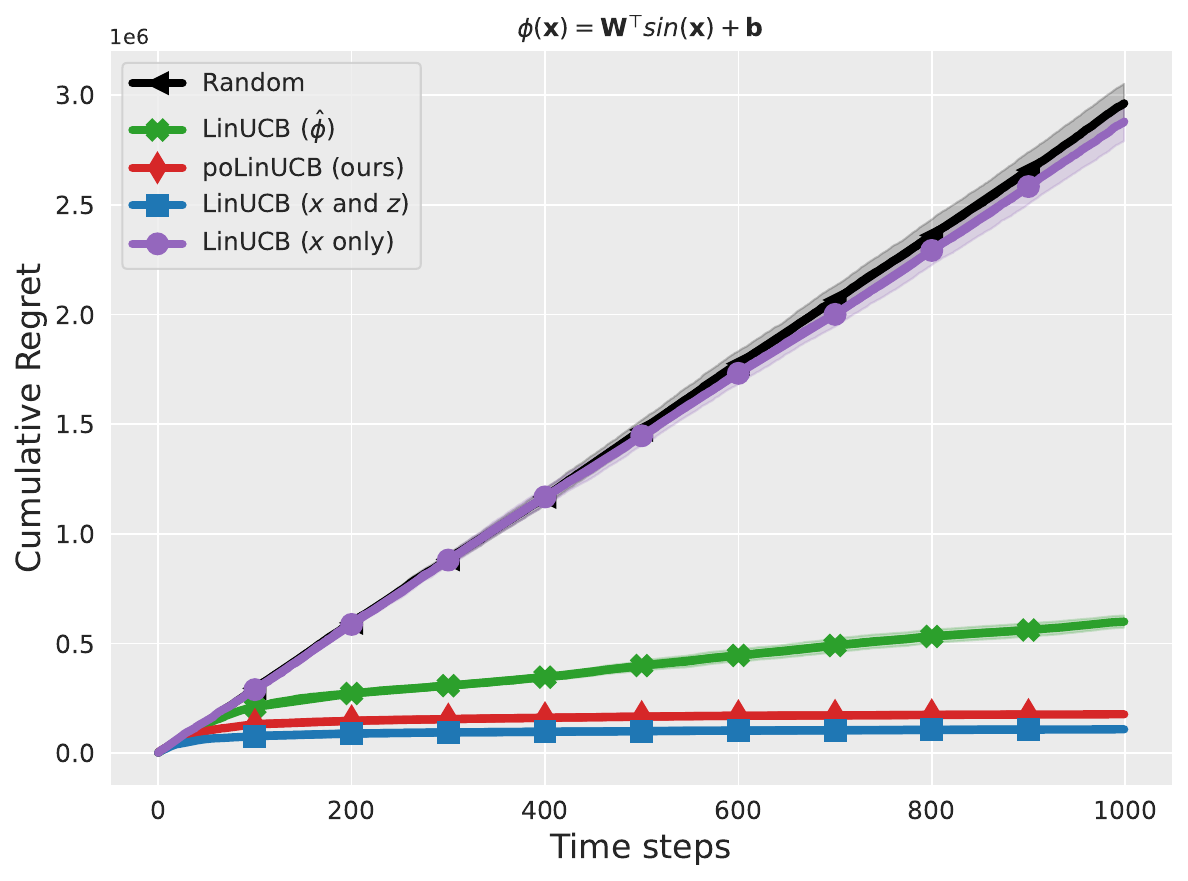}
    \vspace{-0.3cm}
    \caption{Comparison under the setup where the reward's dependency on the post-serving context is noiseless.}
    \label{fig:newcase}
    \vspace{-0.3cm}
\end{figure}

\subsection{Additional Experiments}\label{app:extra-experiments}
We further investigate the case where the true reward is $\langle \phi^\star(\vx),  \vbeta\rangle$. For this case, the reward's dependency on the post-serving context is noiseless. The results are presented in Figure~\ref{fig:newcase}. We observe that the algorithm adapted from \citet{wang2016learning} still performs worse than our method, though the gap seems become smaller. 

\subsection{Additional Discussions on Assumption \ref{assump:learnabilityOnPhiAdvNew}}\label{app:extra-disucssions}
Our regret analysis can accommodate different values of $\alpha$ in Assumption 1. The rate of $1/\sqrt{T}$ (when $\alpha=0.5$) is a commonly observed rate for many classical machine learning algorithms, including linear regression, logistic regression, and SVM with a linear kernel. This rate is rooted in the law of large numbers and the central limit theorem. For smaller $\alpha$ values, the generalization error will converge more slowly than $1/\sqrt{T}$, indicating that the learning problem becomes increasingly difficult. In the extreme case when $\alpha = 0$,  $\phi^\star$ function cannot be learned accurately, thus we will inevitably suffer a linear regret (simply due to model misspecification). 

For standard linear functions, $\alpha = 0.5$ and our regret bound in such situations is tight w.r.t. $\alpha$. However, we intentionally made assumption 1 to be more general in order to accommodate other much more complex ways of estimating the $\phi$ functions such as manifold regression (see, e.g., \citet{yang2016bayesian}), non-parametric ways like k-nearest-neighbors to estimate phi, or to estimate complex non-smooth function (e.g., functions in Holder spaces). For example, when $\phi$ is a function in Holder space $H(\beta)$, then the learning rate is $T^{-2\beta/(2\beta + d)}$, which is generally slower than $T^{-0.5}$ and depends on $\beta$ as well as the data dimension $d$ (see, e.g., the note of \citet{Tibshirani2017nonparametric}). A visual illustration can be found in Figure~\ref{fig:visual-illustraion-holder}.

\end{appendix}

\end{document}